\documentclass[journal]{IEEEtran}
\usepackage{fullpage}
\usepackage{graphicx}
\usepackage{amsmath, amssymb, amsfonts, amsthm}
\usepackage{mathrsfs}
\usepackage{url}
\usepackage{latexsym}
\usepackage[ruled,vlined]{algorithm2e}
\usepackage{color}
\usepackage{cite}
\usepackage{subfigure}
\definecolor{darkblue}{rgb}{0,0,0.4}

\usepackage[colorlinks, linkcolor=red, anchorcolor=blue,citecolor=black]{hyperref}
\newtheorem{thm}{Theorem}[section]

\theoremstyle{remark}
\newtheorem{rem}[thm]{Remark}

\graphicspath{{./figs/}} %note trailing /
\begin{document}

\title{The iterative convolution-thresholding method (ICTM) for image segmentation}
\author{Dong~Wang~and~Xiao-Ping~Wang
\thanks{D. Wang is with the Department of Mathematics, University of Utah, Salt Lake City, UT, 84112 USA. E-mail: dwang@math.utah.edu}
\thanks{X.-P. Wang is with the Department of Mathematics, Hong Kong University of Science and Technology, Hong Kong. E-mail: mawang@ust.hk}
\thanks{Manuscript received , 20xx; revised , 20xx.}
}

%\markboth{Transactions on Image Processing ~Vol.~, No.~, ~201}
%{D. Wang and X.-P. Wang: ICTM}

\maketitle

\begin{abstract}
In this paper, we propose a novel iterative convolution-thresholding method (ICTM) that is applicable to a range of variational models for image segmentation. A variational model usually minimizes an energy functional consisting of a fidelity term and a regularization term. In the ICTM, the interface between two different segment domains  is implicitly represented by their characteristic functions. The fidelity term is then usually written as a linear functional of the characteristic functions and the regularized term is approximated by a functional of characteristic functions in terms of heat kernel convolution.  This allows us to design an iterative convolution-thresholding method to minimize the approximate energy. The method is simple, efficient and enjoys the energy-decaying property. Numerical experiments show that the method is  easy to implement, robust  and applicable to various image segmentation models.
\end{abstract}


\IEEEpeerreviewmaketitle

\section{Introduction}

%\IEEEPARstart{I}{mage} 
Image segmentation is one of the fundamental tasks in image processing. In broad terms, it is the process of partitioning a digital image
into many segments according to a characterization of the image.
The motivation behind this is to determine automatically which part of an image is meaningful for analysis, which also makes it one of the fundamental problems in computer
vision. Many practical applications require image segmentation, like content-based
image retrieval, machine vision, medical imaging, object detection and traffic control systems \cite{mitiche2010variational}.

\medskip

{\it Various models for image segmentation:} Variational methods have enjoyed tremendous success in image segmentation. A typical variational method for image segmentation starts with choosing an energy functional over the space of all segmentations, minimizing which  gives a segmentation with desired properties. For instance, the {\bf  Mumford-Shah model} \cite{mumford1989optimal} uses the following formulation of energy:
\begin{align}\label{MS}
&\mathcal{E}_{MS}(u,\Gamma) \\
=& \lambda\int_D{(u-f)}^2dx + \mu\int_{D\setminus\Gamma}|\nabla u|^2dx + Length(\Gamma) , \nonumber
\end{align}
where $\Gamma$ is  a closed subset of $D$ given by the union of a finite number of curves representing the set of edges (i.e. boundaries of homogeneous regions) in the image $f$, $u$ is a piecewise smooth approximation to $f$, and $ \mu$ and $\lambda $ are positive constants. Despite its descriptiveness, the non-convexity of~(\ref{MS}) makes the minimization problem difficult to analyze and solve numerically \cite{ambrosio1990approximation}. 

To address this issue, a  useful simplification of (\ref{MS}) is to restrict the minimization to functions ({\it i.e.} segmentations) that take a finite number of values. The resulting model is commonly referred to as the {\bf piecewise constant Mumford-Shah model} \cite{chan2001active,vese2002multiphase}.  That is, the $n$ segments $\Omega_i$ ($i \in [n]$) can be obtained by minimizing the following $n$-phase  Chan--Vese (CV) functional~\cite{chan2001active, vese2002multiphase}:
\begin{align}
&\mathcal{E}_{CV}(\Omega_1, \ldots, \Omega_n, C_1, \ldots, C_n) \label{CV} \\
=&\lambda  \sum_{i=1}^n |\partial \Omega_i|+\sum_{i=1}^n \int_{\Omega_i}|C_i-f|^2 \ dx \nonumber
\end{align}
where $\partial \Omega_i$ is the boundary of the $i$-th segment $\Omega_i$, $|\partial \Omega_i|$ denotes the perimeter of the domain $\Omega_i$,  $\lambda$ is a positive parameter, and $C_i$ is the average of the image $f$ within $\Omega_i$ and is defined as follow: $$C_i = \frac{ \int_{\Omega_i} f \ dx }{ \int_{\Omega_i} 1 \ dx }.$$
Here and in the subsequent text, we use the notation $i \in [n]$ to denote $i = 1, 2, \ldots, n$.

When the intensity  inhomogeneity of the image serves, a {\bf  local intensity fitting (LIF) model} \cite{Li_2007,Chunming_Li_2008} was proposed  to overcome the segmentation difficulty caused by intensity inhomogeneity for two-phase problems. Generally, for the $n$-phase problem, the segmentation can be obtained by minimizing the following LIF energy with regularized terms:
\begin{align}
&\mathcal{E}_{LIF}(\Omega_1, \ldots, \Omega_n, C_1, \ldots, C_n) \nonumber \\
=&\lambda  \sum_{i=1}^n |\partial \Omega_i|\label{LBF} \\
+&\sum_{i=1}^n \mu_i \int_{\Omega}\int_{\Omega_i}G_{\sigma}(x-y)|C_i(x)-f(y)|^2 \ dydx \nonumber
\end{align}
where \begin{align}\label{kernel}
G_{\sigma}(x)=\frac{1}{4\pi\sigma} \exp(-\frac{|x|^2}{4\sigma}),
\end{align}
is a two-dimensional Gaussian kernel with standard derivation $\sigma$, $C_i(x)$ are intensity fitting functions, and $\lambda$ and $\mu_i$ are fixed parameters of the model.

Wang et al. \cite{Wang_2009} proposed a model combining the advantages of the CV model \eqref{CV} and the LIF model \eqref{LBF} by taking into account the local and global intensity information. They defined the {\bf  local global intensity fitting (LGIF)} energy functional with regularized terms for the $n$-phase problem:
\begin{align}
&\mathcal{E}_{LGIF}(\Omega_1, \ldots, \Omega_n, C_1, \ldots, C_n, I_1, \ldots, I_n)  \nonumber\\
=&\lambda  \sum_{i=1}^n |\partial \Omega_i|+ \omega \sum_{i=1}^n \int_{\Omega_i}|I_i-f|^2 \ dx \label{LGIF}\\
+&(1-\omega) \sum_{i=1}^n \mu_i \int_{\Omega}\int_{\Omega_i}G_{\sigma}(x-y)|C_i(x)-f(y)|^2 \ dydx  \nonumber
\end{align}
where $\omega$ is a positive constant ($0 \leq \omega \leq 1$), $C_i(x)$ are the intensity fitting functions, and $I_i$ are the average of the image $f$ within $\Omega_i$. Note that LGIF reduces to the CV model when $\omega=1$ and to the LIF model when $\omega=0$.

Recently, several {\bf  locally statistical active contour (LSAC)} models have also been proposed for image segmentation with intensity inhomogeneity. For example, Zhang et al. \cite{zhang2013local} considered the following model of intensity inhomogeneity:
\begin{align}
f(x)= b(x)I(x)+i(x)
\end{align}
where $b(x)$ is the bias field, $I(x)$ is the true signal to be restored, and $i(x)$ is the noise. 
Zhang et al. \cite{zhang2013local} proposed to minimize the following energy functional: 
\begin{align}
&\mathcal{E}_{LSAC}(\Omega_1, \ldots, \Omega_n, C_1, \ldots, C_n, \nu_i, \ldots, \nu_n, \rho, b) \nonumber \\ 
=& \sum_i^n \int_{\Omega_i} \int_\Omega I_\rho(x-y) \label{LSACno} \\
& \left(\log(\nu_i) +|f(x)-b(y)C_i|^2/2\nu_i^2\right)\  dydx  \nonumber
\end{align}
where $\nu_i$ is the standard variance of the noise $i(x)$, 
$$I_\rho(x) = \begin{cases} 1 \ \textrm{if} \ |x|<\rho,  \\ 0  \  \textrm{otherwise}, \end{cases} $$
and $\rho$ is a parameter in the kernel $I_\rho$.
One can also consider the following energy with regularized terms:
\begin{align} 
&\mathcal{E}_{LSAC}(\Omega_1, \ldots, \Omega_n, C_1, \ldots, C_n, \nu_i, \ldots, \nu_n, \rho, b) \nonumber \\ 
=& \lambda  \sum_{i=1}^n |\partial \Omega_i|+ \sum_i^n \int_{\Omega_i} \int_\Omega I_\rho(x-y)\label{LSAC} \\
& \left(\log(\nu_i) +|f(x)-b(y)C_i|^2/2\nu_i^2\right) \ dydx.   \nonumber
\end{align}

\medskip

{\it Existing numerical methods:}
Over the years, various numerical methods have been developed to solve above problems\cite{ambrosio1990approximation,braides1998approximation,tsai2001curve,vese2002multiphase}. For example, instead of solving the optimization problem directly, Bae et al. \cite{bae2011global} solved a dual formulation of the continuous Potts model based on its convex relaxation. Cai et al. \cite{cai2013two} proposed a two-stage segmentation method combining the split Bregman method \cite{goldstein2009split} for finding the minimizer of a convex variant of the Mumford-Shah functional with a K-means clustering algorithm to segment the image into $k$ segments. One of the advantages of this method is that the number of segments does not need to be specified before finding the minimizer.  Dong et al. \cite{dong2010frame} introduced a frame-based model in which the perimeter term was approximated by a term involving framelets.  The framelets were used to capture key features of biological structures. The model can be quickly implemented using the split Bregman method \cite{goldstein2009split}.

The  level-set method has been used by many authors to successfully implement the image segmentation models, which allowed automatic detection of interior contours (see \cite{Wang_2009} and references therein). 
%For instance, if one considers the two phase case, the interface between two segments is represented by the zero level set of a level set function $\varphi$. Then, the perimeter can be approximated by $\int_{\Omega} |\nabla H(\varphi) | dx$ where $H(\varphi)$ is the Heaviside function of $\varphi$ and other terms in above energy functionals can also be represented by $\varphi$. Then, the gradient flow of the energy functional leads to partial differential equations. One can solve the partial differential equations to evolve the level set function to the stationary state which implies the dynamics of the zero level set. 
However, reinitialization is usually needed to keep the level-set function regularized.  In addition,  the method introduces an artificial time step which must be relatively small for stability reasons. It is also difficult to generalize the method to multiphase segmentations.  

A phase-field approximation of the energy was proposed in \cite{esedog2006threshold} for the two-phase CV model,  in which the Ginzburg--Landau functional is used to approximate the perimeter of the domain. The resulting gradient flow, an Allen--Cahn-type equation, can be solved efficiently by many existing methods such as the convex splitting approach. It was also generalized to the Ginzburg-Landau energy functional on graphs using the graph Laplacian for semi-supervised learning models in a series of papers \cite{bertozzi2012diffuse,merkurjevsemi,merkurjev2014graph,garcia2014multiclass}.
%However, it would be very complicated when considering multi-segment images. 

In Wang et al. \cite{wang2016efficient}, we proposed a new iterative thresholding method for the image segmentation based on the multi-phase CV model. In that method, the interfaces between each two segments are implicitly determined by their characteristic functions and the regularized term is written into a nonlocal approximation based on the characteristic functions. Using the coordinate descent method combined with the sequential linear programming, we developed an unconditionally energy-decaying scheme to  solve the multi-phase CV model with an arbitrary number of segments and achieved promising results.

\medskip

{\it Novelty and contributions of this paper:} In this paper, we propose a novel framework that is applicable to a range of models for image segmentation. We consider a rather general energy functional consisting of a fidelity term and a regularized term for the $n$-phase image segmentation problem:
\begin{align}
\mathcal{E} = \sum_{i=1}^n \int_{\Omega_i} F_i(f, \Theta_1, \ldots, \Theta_n) \ dx + \lambda \sum_{i=1}^n |\partial \Omega_i| \label{EnergyGeneral}
\end{align}
where $\Theta_i = (\Theta_{i,1}, \Theta_{i,2}, \ldots, \Theta_{i,m})$ contains all possible variables or functions in fidelity terms.
% ({\it e.g.},  $\Theta_{i,1} = C_i$ and $m=1$ in \eqref{CV};  $\Theta_{i,1} = C_i(x)$,  $\Theta_{i,2} = I_i$, and $m=2$ in \eqref{LGIF}).
The $F_i$ are quite general that will include the models  \eqref{CV}, \eqref{LGIF} and \eqref{LSAC} as special cases.
We then design a novel iterative convolution-thresholding method (ICTM) to minimize the general energy functional \eqref{EnergyGeneral}. We further prove the unconditionally energy-decaying property of the proposed algorithm. The proposed ICTM is simple and easy to implement. Numerical results show that the ICTM converges rapidly and is efficient,  robust  and applicable to a range of models for image segmentation. 

In particular, we compare the performance of  our method with that of the level-set method on several popular image segmentation models in 
%, the proposed algorithm is not sensitive to the number of phases. At each iteration, no regularization and reinitialization of the level-set function and no restrictions on the time step are needed in the ICTM. {\color{red} We performed the numerical experiments to several successful models for image segmentation
 \cite{chan2001active, vese2002multiphase, zhang2013local, Chunming_Li_2008}. 
% and compared the results from the ICTM to those from the level-set method.}
 Numerical results show that the number of iterations needed to reach the stationary state is greatly reduced using  ICTM, compared to those using the level set method.

\medskip

The paper is organized as follows. In Section~\ref{sec:prev}, we review some of the previous numerical methods related to the ICTM. The algorithm is then derived and the energy-decaying property is proved in Section~\ref{sec:der}.  In Section~\ref{sec:num}, we show numerical results to verify the high efficiency of the  ICTM. We then draw conclusions and give a discussion in Section~\ref{sec:con}.

\section{Previous work related to the proposed method}\label{sec:prev}

In 1992, Merriman, Bence, and Osher (MBO) \cite{merriman1992diffusion,merriman1994motion}
developed a threshold dynamics method for the motion of an interface driven
by the mean curvature. 
To be more precise, let $D \subset \mathbb{R}^n$ be a domain whose boundary
$\Gamma= \partial D$ is evolved via motion by mean curvature.
The MBO method is iterative and generates a new interface $\Gamma_{\text{new}}$ (or equivalently $D_{\text{new}}$) at each time step via the following two steps: \\

\noindent {\bf Step 1.} Solve the Cauchy initial value problem for the heat diffusion equation until time $t = \tau$, 
\begin{align*}
& u_t = \Delta u , \\
& u(t=0,\cdot) = \chi_{D},
\end{align*}
where $\chi_D$ is the characteristic function of domain $D$. Let $\tilde u(x) = u(\tau,x)$.\\ 

\noindent  {\bf Step 2.} Obtain a new domain $D_{\text{new}}$ with boundary $\Gamma_{\text{new}} = \partial D_{\text{new}}$  as follows:
\begin{align*}
D_{\text{new}} = \left\{ x\colon \tilde u (x) \geq \frac{1}{2} \right\}.
\end{align*} 

The MBO method has been shown to converge to the continuous motion by mean curvature \cite{evans1993convergence}. 
Esedoglu and Otto  generalized this type of method to multiphase flow with arbitrary surface tensions \cite{esedoglu2015threshold}. The method has attracted much attention thanks to its simplicity and unconditional stability. It has  subsequently been extended to many other applications including the problem of area- or volume-preserving interface motion \cite{ruuth2003simple}, 
image processing \cite{wang2016efficient,esedog2006threshold,merkurjev2013mbo}, 
problems of  anisotropic interface motions \cite{merriman2000convolution,ruuth2001convolution,bonnetier2010consistency,elsey2016threshold}, 
the wetting problem on solid surfaces \cite{xu2016efficient, wang2018efficient}, topology optimization \cite{chen2018efficient}, foam bubbles \cite{wang2018dynamics},
graph partitioning and data clustering  \cite{Gennip2013}, 
and auction dynamics \cite{jacobsauction}. 
Various algorithms and rigorous error analysis have been introduced to refine and extend the original MBO and related methods for the aforementioned problems (see, {\it e.g.}, \cite{deckelnick2005computation,esedoglu2008threshold,merriman1994motion,ishii2005optimal,ruuth1998diffusion,ruuth1998efficient}).
Adaptive methods based on non-uniform fast Fourier transform (NUFFT) \cite{nufft2,nufft6} have also been used to accelerate this type of method \cite{jiang2016nufft}.  Generalized target-valued diffusion-generated methods are recently developed in \cite{wang2018diffusion,osting2017generalized,osting2018}.

\section{Derivation of the method} \label{sec:der}
In this section, we derive the ICTM to minimize \eqref{EnergyGeneral}. 
For simplicity, we first derive the proposed ICTM for the two phase segmentation in Section~\ref{sec:deri2phase}.  The generalization of the method to the multi-segment case is quite straightforward as we show   in Section~\ref{sec:derimphase}, implying that the proposed ICTM is not sensitive to the number of segments. 

\subsection{Derivation of the ICTM for the two-segment case} \label{sec:deri2phase}
For simplicity, we describe the ICTM in the case of two-phase segmentation.  The ICTM is a region-based method.
%We note that the goal in image segmentation is to find interfaces between different partitions to recognize different segments. 
%It is important to determine how to represent the interface before designing the method to find the optimal segmentation in the sense of minimizing the corresponding energy. 
%In the image where only two segments are considered. To derive the ICTM, we denote $u(x)$ as the characteristic function of 
In our method,  the first segment $\Omega_1$ is denoted by its characteristic function $u(x)$, {\it i.e.},
\begin{align}
u(x) := \begin{cases} 1 \ \textrm{if} \ x \in \Omega_1, \\ 
0 \ \textrm{otherwise}.\end{cases} 
\end{align}
Then the characteristic function of the second segment $\Omega_1^c$ is $1-u(x)$. Note that the interface between two segments is now implicitly represented by $u(x)$. 

As pointed by Esedoglu and Otto \cite{esedoglu2015threshold}, when $\tau \ll 1$, the length  of $\partial \Omega_1$ can be approximated by
\begin{align}
|\partial \Omega_1|\approx \sqrt{\frac{\pi}{\tau}}\int_{\Omega} u G_{\tau} *(1-u) \ dx, \label{approxPeri}
\end{align}
where $*$ represents convolution and $G_\tau$ is defined in \eqref{kernel}.
%Intuitively, the above integral measures the amount of heat that escapes from $\Omega_1$ to $\Omega_2$ which can be used to estimate the measure of the boundary between $\Omega_1$ and $\Omega_2$ after normalization. 
The rigorous proof of the convergence as $\tau \searrow 0$ can be found in Miranda et al. \cite{Miranda_2007}. 

The fidelity terms in $\mathcal{E}$ can then be written into an integral on the whole domain $\Omega$ by multiplying the integrand by $u$ or $1-u$. That is,
\begin{align*}\int_{\Omega_1} F_1 \ dx  = \int_{\Omega} u F_1 \ dx, \ 
\int_{\Omega_2} F_2\ dx  & = \int_{\Omega} (1-u) F_2 \ dx.
\end{align*}
Hence the total energy \eqref{EnergyGeneral} can be approximated by
\begin{align}
\mathcal{E} \approx \mathcal{E}^{\tau}(u,\Theta) \colon =\mathcal{E}_f (u,\Theta) +\mathcal{E}_r^{\tau} (u,\Theta) \label{energy12}
\end{align}
where 
$$\mathcal{E}_f (u,\Theta) =  \int_{\Omega} u F_1(f,\Theta)+(1-u) F_2(f,\Theta) \ dx$$ and  $$\mathcal{E}_r^{\tau}(u,\Theta)  = \lambda \sqrt{\frac{\pi}{\tau}}  \int_{\Omega} u G_{\tau} *(1-u) \  dx.$$
The $\Gamma$ convergence of $\mathcal{E}^{\tau}$ to $\mathcal{E}$ when $\tau \searrow 0$ can be proved similar to that in Esedoglu and Otto \cite{esedoglu2015threshold} or Wang et al. \cite{wang2018efficient} and thus the solution for the segmentation can be approximated by finding $u^{\tau,\star}$ such that
\begin{align}
(u^{\tau, \star}, \Theta^{\tau,\star}) = \arg\min_{u \in \mathcal{B}, \Theta \in \mathcal{S}} \mathcal{E}^\tau(u,\Theta)
\end{align} where 
\begin{align*}
\mathcal{B} \colon = \{u\in BV(\Omega,\mathbb{R}) \ |  \ u =\{0,1\} \}
 \end{align*} 
 and $BV(\Omega,\mathbb{R})$ denotes the bounded-variation functional space.

Now, we apply the coordinate descent method to minimize $\mathcal{E}^\tau(u,\Theta)$; that is, starting from an initial guess: $u^0$, we find the minimizers iteratively in the following order:
\[\Theta^0,u^1, \Theta^1, \ldots, u^k, \Theta^k, \ldots .\]
Without loss of generality, assuming that $u^k$ is calculated, we fix $u^k$ and find the minimizer of $\mathcal{E}^\tau(u^k,\Theta)$ to obtain $\Theta^k$. That is,
\begin{align}
\Theta^k & = \arg\min_{\Theta \in \mathcal{S}}  \mathcal{E}^\tau(u^k,\Theta).\label{min:theta}
\end{align}
Here and in the subsequent sections, we generally assume that for the $n$-phase case, the global minimizer of \[\sum_{i=1}^n \int_{\Omega_i} F_i(f, \Theta_1, \ldots, \Theta_n) dx \] exists and is unique on $\mathcal{S} =\mathcal{S}_1 \times\mathcal{S}_2\times \ldots\times \mathcal{S}_n $ where $S_i$ are the admissible sets for $\Theta_i$.
\begin{rem}
This assumption is reasonable for models for image processing because most of these models use strictly convex fidelity terms, such as those  in \eqref{CV}, \eqref{LBF}, \eqref{LGIF}, and \eqref{LSAC}.
\end{rem}

Because $\mathcal{E}_r^\tau$ is independent of $\Theta$, one only needs to find the global minimizers of $\mathcal{E}_f$ with respect to $\Theta$ to obtain $\Theta^k$. That is, 
\begin{align}
\Theta^k & = \arg\min_{\Theta \in \mathcal{S}}  \mathcal{E}_f (u^k,\Theta)\label{min:theta1} \\
& = \arg\min_{\Theta \in \mathcal{S}} \int_{\Omega} u^k F_1(f,\Theta)+(1-u^k) F_2(f,\Theta) \ dx. \nonumber
\end{align} 
This optimization problem can be solved in different ways for different types of functionals. For example, if $\mathcal{E}_f  $ is strictly convex and differentiable with respect to each element in $\Theta$, then each element  $\Theta_{i,j}$ ($i=1, 2, j \in [m]$) in $\Theta^k$ can be obtained via solving the following system of equations: 
\begin{equation} \label{syseqn}
\begin{cases}
\dfrac{\partial \mathcal{E}_f}{\partial\Theta_{1,1}} = 0,  \ \ldots, \ 
\dfrac{\partial \mathcal{E}_f}{\partial\Theta_{1,m}} = 0,  \\
\dfrac{\partial \mathcal{E}_f}{\partial\Theta_{2,1}} = 0,   \  \ldots,  \
\dfrac{\partial \mathcal{E}_f}{\partial\Theta_{2,m}} = 0.    \\
\end{cases}
\end{equation}
\begin{rem}
We use the notation $\frac{\partial (\cdot) }{\partial \cdot}$ to denote either variation (when $\Theta_{i,j}$  are scalar functions)  or derivative (when  $\Theta_{i,j}$ are  scalar variables). Then, \eqref{syseqn} can be efficiently solved using the Gauss--Seidel strategy similar to that in \cite{wang_2001} (see  examples in Section~\ref{sec:num}).
\end{rem} 

After solving $\Theta^k$, we then solve  $u^{k+1}$ by
\begin{align}
u^{k+1} = \arg\min_{u \in \mathcal{B}}  \mathcal{E}^\tau(u,\Theta^k). \label{min:u}
\end{align}

Note that the set $\mathcal{B}$ contains the  boundary points of the following convex set $\mathcal{K}$:
\begin{align*}
\mathcal{K} =& \{u\in BV(\Omega,\mathbb{R}) \ | \  u \in [0,1] \}.
\end{align*}
In other words, $\mathcal{K}$ is the convex hull of $\mathcal{B}$.

When $\Theta^k$ is fixed, it is easy to check that  $\mathcal{E}^\tau(u,\Theta^k)$ is a concave functional because $\mathcal{E}_f (u,\Theta^k)$ is linear and $\mathcal{E}_r^\tau(u,\Theta^k)$ is concave. Using the fact that the minimizer of a concave functional on a convex set can only be attained at the boundary points of the convex set and by finding a minimizer on a convex set $\mathcal{K}$, we relax the original problem \eqref{min:u} to the following equivalent problem \eqref{min:urelax}:
\begin{align}
u^{k+1} = \arg\min_{u \in \mathcal{K}}  \mathcal{E}^\tau(u,\Theta^k).  \label{min:urelax}
\end{align}
%The following lemma guarantees the equivalence between the minimization problem \eqref{min:urelax} and the minimization problem \eqref{min:u}.
%
%\begin{lem} \label{lem:eq}
%For any fixed $\Theta$ and $\mathcal{E}^\tau(u,\Theta)$ defined in \eqref{EnergyGeneral}, we have
%\[\arg\min_{u \in \mathcal{K}}  \mathcal{E}^\tau(u,\Theta) =\arg\min_{u \in \mathcal{B}}  \mathcal{E}^\tau(u,\Theta) . \]
%\end{lem}
%\begin{proof}
%See Appendix \ref{append1}.
%\end{proof}

The sequential linear programming then leads to the following linearized problem:
\begin{align}
u^{k+1} = \arg\min_{u \in \mathcal{K}}  \mathcal{L}^{\tau}(f,\Theta^k,u^k,u) \label{min:ulinear}
\end{align}
where $\mathcal{L}^{\tau}(f,\Theta^k,u^k,u)$ is the linearization of $\mathcal{E}^\tau(u,\Theta^k)$ at $u^k$,
\begin{align}
&\mathcal{L}^{\tau}(f,\Theta^k,u^k,u)\colon   =  \int_{\Omega} u \phi \ dx  \label{min:ulinear}\end{align}
and $$
\phi =  F_1(f,\Theta^k)-F_2(f,\Theta^k) + \lambda \sqrt{\frac{\pi}{\tau}} G_{\tau} *(1-2u^k). \nonumber
$$

After the above relaxation and linearization, the optimization problem \eqref{min:u} is approximated by minimizing a linear functional over a convex set. Because $u(x)\geq 0$, it can be carried out in a pointwise manner by checking whether $\phi(x)>0$ or not.  That is, the minimum can be attained at
\begin{align}
u^{k+1}(x) = \begin{cases} 1 \ \  \textrm{if}  \  \phi(x)\leq 0,   \\
0 \ \  \textrm{otherwise}. \label{min1}
 \end{cases}
\end{align}

Now, combining \eqref{min:theta1} and \eqref{min1} yields Algorithm~\ref{a:MBO}.
\begin{algorithm}[ht]
\DontPrintSemicolon
 \KwIn{Let $\Omega$ be the image domain, $f$ be the image, $\tau > 0$, and $u^0 \in \mathcal{B}$.}
 \KwOut{A scalar function $u^s \in \mathcal{B}$ that approximately minimizes   \eqref{EnergyGeneral}.}
 Set $k=1$\;
 \While{not converged}{
{\bf 1.} For the fixed $u^k$, find
 \[\Theta^k = \arg\min_{\Theta \in \mathcal{S}} \int_{\Omega} u^k F_1(f,\Theta)+(1-u^k) F_2(f,\Theta) \ dx.\]
{\bf 2.} Use $\Theta^k$ from Step 1 and evaluate 
\[\phi^k(x) = F_1(f,\Theta^k)-F_2(f,\Theta^k) + \lambda \sqrt{\frac{\pi}{\tau}} G_{\tau} *(1-2u^k).\]
{\bf 3.} Set
\[u^{k+1}(x) = \begin{cases}1 \ \  \textrm{if}  \ \phi^k(x)\leq 0 ,   \\
0 \ \  \textrm{otherwise}.  \end{cases}\]
Set $k = k+1$\;
 }
\caption{An iterative convolution-thresholding method (ICTM) for approximating minimizers of the energy in \eqref{EnergyGeneral}. }
\label{a:MBO}
\end{algorithm}
\begin{rem}
In Theorem~\ref{thm:stability} below, we will prove that Algorithm~\ref{a:MBO} is unconditionally stable for any $\tau>0$. Since we are using characteristic functions to implicitly represent the interface between two segments, the criterion on the convergence of Algorithm~\ref{a:MBO} is $ \int_\Omega |u^k-u^{k-1}| \ dx< tol $ for a relatively small value of $tol$. In practice, because the image is defined in a discrete domain, the criterion for the convergence is that no pixel switches from one segment to the other between two iterations. 
\end{rem}

Theorem~\ref{thm:stability} below shows that  the total energy $\mathcal{E}^{\tau}(u,\Theta)$ decreases in the iteration for any $\tau > 0$. Therefore,
our iteration algorithm always converges to a stationary partition for any initial partition.

\begin{thm}[Stability]\label{thm:stability}
Let $(u^k,\Theta^k)$ be the $k$-th iteration derived
in Algorithm~\ref{a:MBO}. We have
\[\mathcal{E}^\tau(u^{k+1},\Theta^{k+1}) \leq \mathcal{E}^\tau(u^{k},\Theta^{k}) \]
for any $\tau$.
\end{thm}
\begin{proof}
See Appendix \ref{append2}.
\end{proof}

%\begin{rem}[Convergence]
As we will show by numerical examples in Section~\ref{sec:num}, the ICTM converges very fast and the number of iterations for convergence  is greatly reduced. One can understand this advantage of the ICTM as the follows: 
%\begin{itemize}
 The approximate energy functional \eqref{energy12} is the summation of a strictly convex functional (or, more generally, a functional with a global minimizer) with respect to $\Theta$ ({\it i.e.}, $\mathcal{E}_f$) and a concave  functional only dependent on $u$ ({\it i.e.}, $\mathcal{E}_r^\tau$). At the first step, $\Theta^k$ is the optimal choice to decrease the energy.  At the second and the third step, we find the minimizer of the linear approximation which is also the optimal choice to minimize the linearized functional. Moreover, the minimizer can give a smaller value in \eqref{energy12} because the graph of the functional $\mathcal{E}_r^\tau$ is always below its linear approximation. This accelerates the convergence of the ICTM.

%\item In the level-set method, the first step can be done in the same way as the ICTM using the Heaviside function of the level-set function.
%At the second step, one needs to solve the level-set equation with a relatively small time step. This step more or less restricts the decay of the energy (at least not optimal). What's worse, the reinitialization step after it (or adding penalty terms in the level-set equation) usually increases the energy, which decreases the value of the energy minimized in first two steps at each iteration. This makes the level-set method converge slower. In the ICTM, thanks to the concavity of $\mathcal{E}_r^\tau$, the minimizer at each iteration automatically gives a new partition ({\it i.e.}, the minimizer is automatically attained on the set $\mathcal{B}$). No reinitialization and related techniques are need in the ICTM.

%\item The ICTM uses the alternate direction minimizing method combined with the Gauss--Seidel iteration to minimize the energy.  Both two methods are efficient and converge fast. 
%That is, we use the fixed $u^k$ to obtain $\Theta^k$ and then use the $\Theta^k$ solved to obtain $u^{k+1}$. 

%\end{itemize}

%
%\end{rem}

\subsection{Derivation of the ICTM for the multi-segment case} \label{sec:derimphase}

To derive the ICTM for the $n$-segment case, we use $n$ characteristic functions and define
\begin{align}
u_i(x) = \chi_{\Omega_i}(x):= \begin{cases} 1 \ \textrm{if} \ x \in \Omega_i, \\ 
0 \ \textrm{otherwise},\end{cases}  \  i \in [n].
\end{align}
Then, we denote  $u = (u_1,u_2, \ldots, u_n)$ and define 
\begin{align*}
\tilde{\mathcal{B}} = \{& u\in BV(\Omega,\mathbb{R}^n) \ |\  u_i =\{0,1\}, \ i \in [n],  \\
& \ \textrm{and} \ \sum_{i=1}^n u_i(x) = 1 \}.
 \end{align*} 
In the $n$-segment case, similar to \eqref{approxPeri}, the measure of $\partial \Omega_i \cap \partial \Omega_j$ can be approximated by
\begin{align}
|\partial \Omega_i \cap \partial \Omega_j|\approx \sqrt{\frac{\pi}{\tau}}\int_{\Omega} u_i G_{\tau} *u_j \ dx, \nonumber
\end{align}
and thus the perimeter of $\Omega_i$ is approximated by 
\begin{align}
|\partial \Omega_i |\approx \sqrt{\frac{\pi}{\tau}} \sum\limits_{ j=1,j\neq i}^{n}  \int_{\Omega} u_i G_{\tau} *u_j \ dx.
\end{align}
Then, the total energy \eqref{EnergyGeneral} can be approximated by
\begin{align}
\mathcal{E}^{\tau}(u,\Theta) =\mathcal{E}_f(u,\Theta) +\mathcal{E}_r^{\tau} (u,\Theta) \label{energy1}
\end{align}
where 
$$\mathcal{E}_f(u,\Theta) =  \sum\limits_{i=1}^n \int_{\Omega} u_i F_i(f,\Theta) \ dx$$ and  $$\mathcal{E}_r^{\tau}(u,\Theta)  = \lambda \sqrt{\frac{\pi}{\tau}} \sum\limits_{i=1}^n \sum\limits_{ j=1,j\neq i}^{n} \int_{\Omega} u_i G_{\tau} *u_j \  dx.$$

Again, we apply the coordinate descent method to minimize $\mathcal{E}^\tau(u,\Theta)$; that is, starting from an initial guess $u^0$, we find the minimizers iteratively in the following order:
\[\Theta^0,u^1, \Theta^1, \ldots, u^k, \Theta^k, \ldots .\]
When $u^k$ is fixed, $\Theta^k$ can be obtained via 
\begin{align}
\Theta^k = \arg\min_{\Theta \in \mathcal{S}}  \mathcal{E}_f (u^k,\Theta). \label{min:theta}
\end{align}
Using the same relaxation and linearization procedure as  in Section~\ref{sec:deri2phase}, we arrive at 
\begin{align}
u^{k+1} = \arg\min_{u \in \tilde{\mathcal{K}}}  \mathcal{L}^{\tau}(f,\Theta^k,u^k,u) \label{min:ulinear}
\end{align}
where 
\begin{align}
&\mathcal{L}^{\tau}(f,\Theta^k,u^k,u)   \label{min:ulinear} \\
 = &  \sum_{i=1}^n \int_{\Omega} u_i \left[ F_i(f,\Theta^k) + 2 \lambda \sqrt{\frac{\pi}{\tau}} \sum\limits_{ j=1,j\neq i}^{n}G_{\tau} *u_j^k \right] \ dx  \nonumber\\
 = &  \sum_{i=1}^n \int_{\Omega} u_i  \phi_i^k \ dx   \nonumber
\end{align}
is a linear functional and 
\begin{align*}
\tilde{\mathcal{K}} =& \{(u_1,u_2, \ldots, u_n)\in BV(\Omega,\mathbb{R}^n) \ | \ u_i \in [0,1],  \\ & \ i \in [n],  \ \textrm{and} \ \sum_{i=1}^n u_i(x) = 1 \}
\end{align*}
is the convex hull of $\tilde{\mathcal{B}}$.
Then, the minimum is attained at
\begin{align}
u_i^{k+1}(x) = \begin{cases} 1 \ \  \textrm{if}  \  \  i = \arg\min_{\ell \in [n]} \phi_\ell^k,   \\
0 \ \  \textrm{otherwise}. \label{min}
 \end{cases}
\end{align}
\begin{rem}
Note that in \eqref{min}, $\arg\min_{\ell \in [n]} \phi_\ell^k$ may have more than one solution. In this case, we simply set $ i = \min \{\arg\min_{\ell \in [n]} \phi_\ell^k \}$.
\end{rem}

Now, we have Algorithm~\ref{a:MBO2} below which is applicable to cases with an arbitrary number of segments and we have Theorem~\ref{thm:stability2} which is same as Theorem~\ref{thm:stability} in Section~\ref{sec:deri2phase} above to guarantee  that the total energy $\mathcal{E}^{\tau}(u,\Theta)$ decreases in the iteration for any $\tau > 0$. Therefore, the ICTM always converges to a stationary partition for any initial partition and an arbitrary number of segments.

\begin{algorithm}[ht]
\DontPrintSemicolon
 \KwIn{Let $\Omega$ be the image domain, $f$ be the image, $\tau > 0$, and $u^0 \in \mathcal{B}$.}
 \KwOut{ A vector-valued function $u^s \in \mathcal{B}$ that approximately minimizes   \eqref{EnergyGeneral}.}
 Set $k=1$\;
 \While{not converged}{
{\bf 1.} For the fixed $u^k$, find
 \[\Theta^k = \arg\min_{\Theta \in \mathcal{S}}\sum_{i=1}^n \int_{\Omega} u_i F_i(f, \Theta) dx.\]
{\bf 2.} For $i\in [n]$, evaluate 
\[\phi_i^k = F_i(f,\Theta^k) + 2\lambda \sum\limits_{ j=1,j\neq i}^{n} \sqrt{\frac{\pi}{\tau}} G_{\tau} *u_j^k.\]
{\bf 3.} For $i\in [n]$, set
\[u_i^{k+1}(x) = \begin{cases}1 \ \  \textrm{if}  \  \  i =\min\{ \arg\min_{\ell \in [n]} \phi_\ell^k\},   \\
0 \ \  \textrm{otherwise}.  \end{cases}\]
Set $k = k+1$\;
 }
\caption{An iterative convolution-thresholding method (ICTM) for approximating minimizers of the energy in \eqref{EnergyGeneral}. }
\label{a:MBO2}
\end{algorithm}

\begin{rem}
The ICTM for the case with two segments is a special case of Algorithm~\ref{a:MBO2}. Also, the ICTM for multiple segments is almost identical to the ICTM for two segments. Similarly, the criterion on the convergence of Algorithm~\ref{a:MBO2} is $ \sum_{i=1}^n \int_\Omega |u_i^k-u_i^{k-1}| \ dx< tol$. In practice, the criterion for the convergence is that no pixel switches from one segment to another between two iterations. 

\end{rem}

\begin{thm}[Stability]\label{thm:stability2}
Let $(u^k,\Theta^k)$ be the $k$-th iteration derived
in Algorithm~\ref{a:MBO2}. We have
\[\mathcal{E}^\tau(u^{k+1},\Theta^{k+1}) \leq \mathcal{E}^\tau(u^{k},\Theta^{k}) \]
for any $\tau$.
\end{thm}
\begin{proof}
See Appendix \ref{append3}.
\end{proof}

\section{Numerical Experiments}\label{sec:num}
In this section, we demonstrate the efficiency of the proposed algorithms by
numerical examples. We implemented the algorithms in MATLAB  installed on a laptop with a 2.7GHz Intel Core i5 processor and 8GB of RAM. We apply our methods to different models and also compare our results with those obtained from the level-set methods in Li et al. \cite{Chunming_Li_2008} and Zhang et al. \cite{zhang2013local}.  Our results show a clear advantage of ICTM  in terms of simplicity and efficiency.
 
\subsection{Applications to the Chan--Vese model (CV) \eqref{CV}} \label{subsec:cv}
The first application of the proposed ICTM is to recover the scheme in Wang et al. \cite{wang2016efficient} for the CV model. Specifically, in \eqref{CV}, the corresponding $F_i(f,\Theta_1, \Theta_2, \ldots, \Theta_n) =  |C_i-f|^2$, $\Theta_i = C_i$, $S_i = \mathbb{R}$, and $S_i^o = \mathbb{R}$  for $i  \in [n]$.

In  Step 1 in Algorithm~\ref{a:MBO}, when $u^k$ is fixed, $$\int_\Omega u|C_1-f|^2+(1-u)|C_2-f|^2 \ dx$$ is strictly convex with respect to $C_1$ and $C_2$. Hence, direct calculation of the stationary points yields
\[C_1^k = \dfrac{\int_\Omega u^k f \ dx}{\int_\Omega u^k \ dx}, \  \ C_2^k = \dfrac{\int_\Omega (1-u^k) f \ dx}{\int_\Omega 1-u^k \ dx},\]
which are  the average intensities of the image $f$ in $\Omega_1$ and $\Omega_2$, respectively.

For the $n$-phase case in  Algorithm~\ref{a:MBO2}, in Step 1, when $u^k$ is fixed, $\sum_{i=1}^n\int_\Omega u_i^k|C_i-f|^2\ dx$ is strictly convex with respect to $C_i$, $i\in[n]$. Hence, the minimizer is given by 
\[C_i^k = \dfrac{\int_\Omega u_i^k f \ dx}{\int_\Omega u_i^k \ dx},\]
which are the average intensities of the image $f$ in $\Omega_i$.
They are all consistent with the definition of $C_i$ in the CV model \eqref{CV}.
Then, in both Algorithm~\ref{a:MBO} and Algorithm~\ref{a:MBO2}, using $C_i^k$ and $u_i^k$, one can calculate $\phi^k$ (or $\phi_i^k$ in Algorithm~\ref{a:MBO2})  with heat kernel convolution using the fast Fourier transform (FFT),  followed by the thresholding step ({\it i.e.}, Step 3) to obtain $u^{k+1}$. This exactly recovers the scheme we derived in Wang et al. \cite{wang2016efficient}. We show examples from \cite{wang2016efficient}, where  more numerical experiments on the CV  can also be found. 

In Figure~\ref{fig:cv}, we show the results of the ICTM applied to the classic flower image. The figures are initial contour, final contour, and final segments from left to right. In the first row, we use Algorithm~\ref{a:MBO} to have the two phase segmention of the image and in the second row, we use Algorithm~\ref{a:MBO2} to obtain four phase segmention of the image.  In this simulation, we set the domain of the image to be $[-\pi,\pi]\times [-\pi,\pi]$ and the convolutions are efficiently evaluated using FFT. The parameters $(\tau, \lambda)$ are $(0.02,0.05)$ and $(0.02, 0.02)$ and the numbers of iterations are $15$ and $14$. The code for the ICTM on the CV model can be downloaded from \url{https://www.math.utah.edu/~dwang/ICTM_CV.zip}. The results show that the ICTM converges to the stationary solutions in very few steps.

\begin{figure*}
\centering
\includegraphics[width=0.2\textwidth,clip,trim= 4cm 2cm 4cm 2cm]{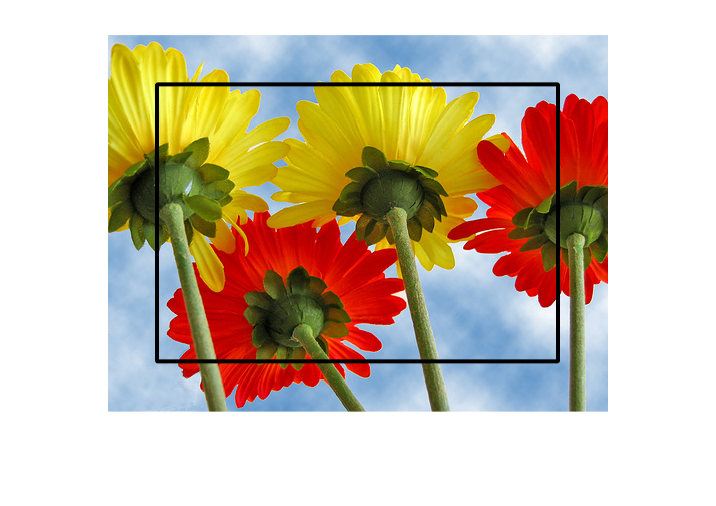} \ \ 
\includegraphics[width=0.2\textwidth,clip,trim= 4cm 2cm 4cm 2cm]{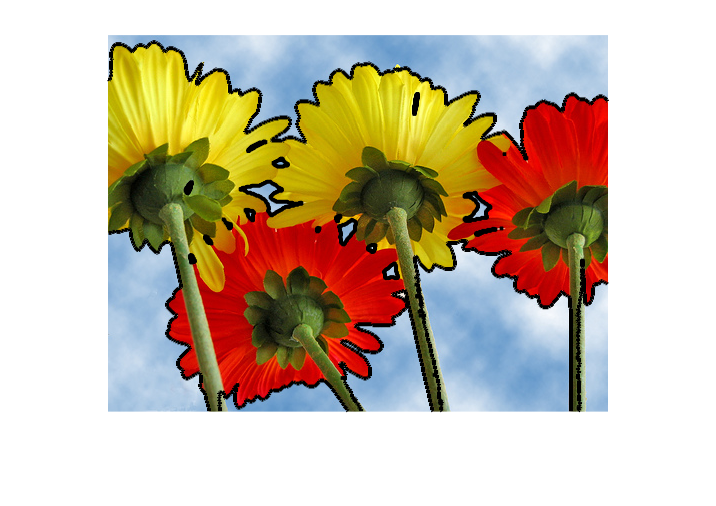}  \ \ 
\includegraphics[width=0.2\textwidth,clip,trim= 4cm 2cm 4cm 2cm]{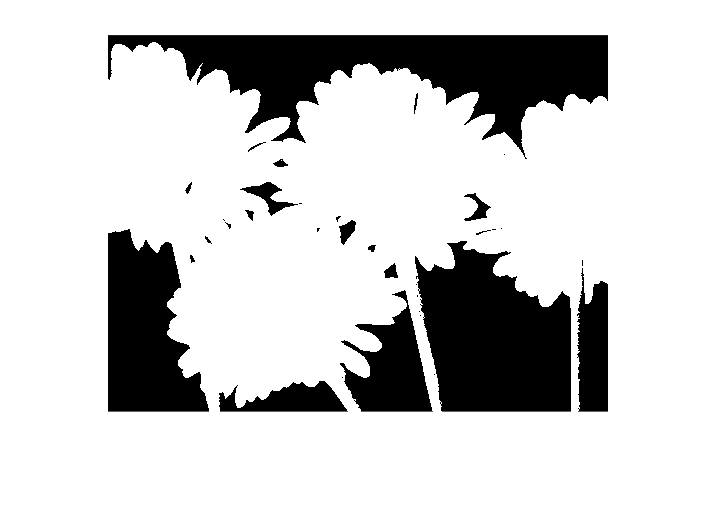} \\ 
\includegraphics[width=0.2\textwidth,clip,trim= 4cm 2cm 4cm 2cm]{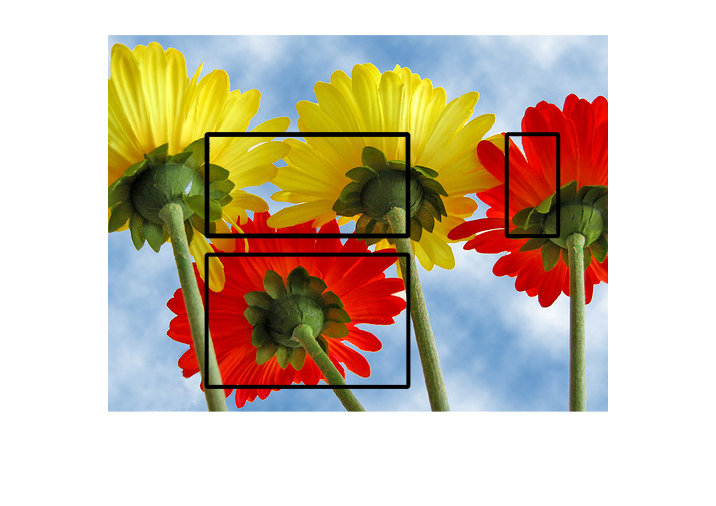}  \ \ 
\includegraphics[width=0.2\textwidth,clip,trim= 4cm 2cm 4cm 2cm]{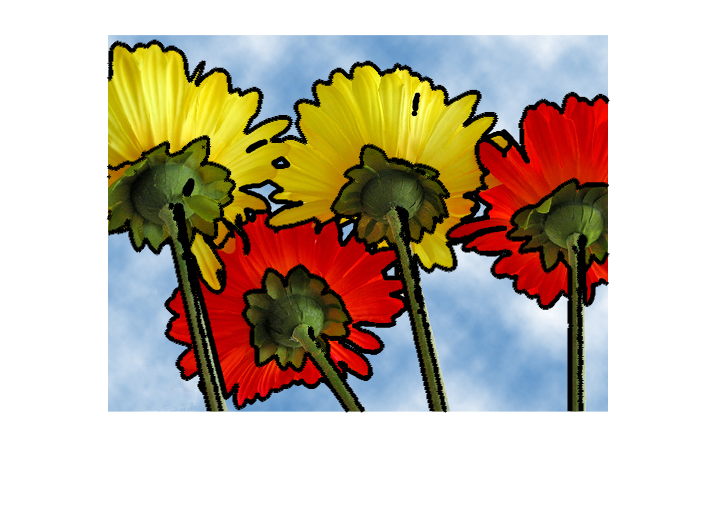}   \ \
\includegraphics[width=0.2\textwidth,clip,trim= 4cm 2cm 4cm 2cm]{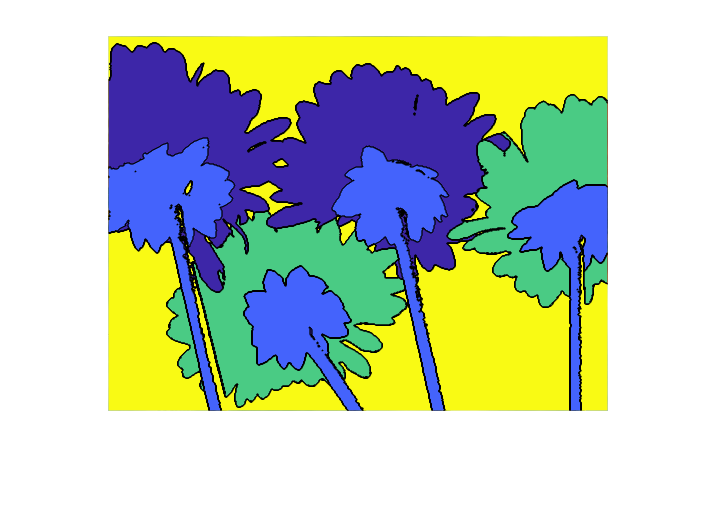} 
\caption{First row: two-phase segmentation with $(\tau, \lambda)=(0.02,0.05)$. The number of iterations is $15$. Second row: four-phase segmentation with $(\tau, \lambda)=(0.02,0.02)$. The number of iterations is $14$. From left to right: initial contour, final contour, and final segments. The code for the ICTM can be downloaded from \url{https://www.math.utah.edu/~dwang/ICTM_CV.zip}. See Section~\ref{subsec:cv} for details.}\label{fig:cv}
\end{figure*}

\subsection{Applications to the locally statistical active contour (LSAC) \eqref{LSAC}}\label{subsec:1}
In this section, we use two-phase  segmentation examples to demonstrate the efficiency of the ICTM. The $n$-phase case can be implemented in a similar way. We now apply the proposed ICTM to the LSAC model \eqref{LSAC}. That is, we choose \begin{align*} &F_i(f,\Theta_1,\Theta_2, \ldots, \Theta_n) \\
=& \int_\Omega I_\rho(x-y) \left(\log(\nu_i) +|f(x)-b(y)C_i|^2/2\nu_i^2\right)  \ dy\end{align*} and $\Theta_i = (\nu_i,b(x), C_i)$ for any $i\in[2]$. Direct calculation shows that the global minimizer of $$ \int_{\Omega} u^k F_1(f,\Theta_1,\Theta_2 )+  (1-u^k) F_2(f,\Theta_1,\Theta_2 )\ dx$$ occurs at its unique stationary point. Since  $F_i$ is independent of $\nu_j$ and $C_j$ for $j \neq i$, Step 1 in Algorithm~\ref{a:MBO} is simplified to
\begin{equation} 
\footnotesize
\begin{cases}
 \iint_{\Omega} u^k(x) I_\rho(x-y)b(y)[f(x) -b(y)C_1] \ dydx = 0, \\
 \iint_{\Omega} (1-u^k(x)) I_\rho(x-y)b(y)[f(x) -b(y)C_2] \ dydx = 0, \\
  \iint_{\Omega} u^k(x) I_\rho(x-y)[\nu_1^2- [f(x) -b(y)C_1]^2] \ dy dx = 0, \\ 
  \iint_{\Omega} (1-u^k(x)) I_\rho(x-y)[\nu_n^2- [f(x) -b(y)C_2]^2] \ dy dx = 0, \\
 \iint_{\Omega} u^k(x) I_\rho(x-y)[f(x)-b(y)C_1]C_1/\nu_1^2 \ dydx \\
 + \iint_{\Omega} (1-u^k(x)) I_\rho(x-y)[f(x)-b(y)C_2]C_2/\nu_2^2 \ dydx = 0.
\end{cases}
\end{equation}
Then, one can use the Gauss--Seidel strategy to obtain  $	(\nu_i,b,C_i)$ for $i\in[2]$:
\begin{equation} \label{solutionTheta1}
\scriptsize
\begin{cases}
C_1^k = \dfrac{\int_\Omega (I_\rho*{b^{k-1}}) f u^k \ dx}{\int_\Omega (I_\rho*{b^{k-1}}^2) u^k \ dx}, \\
C_2^k = \dfrac{\int_\Omega (I_\rho*{b^{k-1}}) f (1-u^k) \ dx}{\int_\Omega (I_\rho*{b^{k-1}}^2) (1-u^k) \ dx}, \\
\nu_1^k =\sqrt{ \dfrac{\int_\Omega \int_\Omega  I_\rho(x-y) u^k(x) (f(x)-b^{k-1}(y)C_1^k)^2 \ dydx   }{\int_\Omega \int_\Omega  I_\rho(x-y) u^k(y)  \ dydx }},\\
\nu_2^k =\sqrt{ \dfrac{\int_\Omega \int_\Omega  I_\rho(x-y) (1-u^k(x)) (f(x)-b^{k-1}(y)C_2^k)^2 \ dydx   }{\int_\Omega \int_\Omega  I_\rho(x-y) (1-u^k(y))  \ dydx }},\\
b^k(x) = \dfrac{[C_1^k/(\nu_1^k)^2] I_\rho*(fu^k)+[C_2^k/(\nu_2^k)^2] I_\rho*(f(1-u^k))}{ [(C_1^k/\nu_1^k)^2]I_\rho*u^k+ [(C_2^k/\nu_2^k)^2]I_\rho*(1-u^k)} .
\end{cases}
\end{equation}

%\begin{equation} 
%\begin{cases}
% \iint_{\Omega} u_1^k(x) I_\rho(x-y)b(y)[f(x) -b(y)C_1] \ dydx = 0, \\
% \qquad \qquad \vdots  \\
% \iint_{\Omega} u_n^k(x) I_\rho(x-y)b(y)[f(x) -b(y)C_n] \ dydx = 0, \\
%  \iint_{\Omega} u_1^k(x) I_\rho(x-y)[\nu_1^2- [f(x) -b(y)C_1]^2] \ dy dx = 0, \\ 
%\qquad \qquad \vdots \\
%  \iint_{\Omega} u_n^k(x) I_\rho(x-y)[\nu_n^2- [f(x) -b(y)C_n]^2] \ dy dx = 0, \\
%\sum_{i=1}^n \iint_{\Omega} u_i^k(x) I_\rho(x-y)[f(x)-b(y)C_i]C_i/\nu_i^2 \ dydx = 0.
%\end{cases}
%\end{equation}
%Then, one can use the Gauss--Seidel strategy to obtain  $	(\nu_i,b,C_i)$ for $i\in[n]$:
%\begin{equation} \label{solutionTheta1}
%\begin{cases}
%C_i^k = \dfrac{\int_\Omega (I_\rho*{b^{k-1}}) f u_i^k \ dx}{\int_\Omega (I_\rho*{b^{k-1}}^2) u_i \ dx}, \\
%\nu_i^k =\\
% \sqrt{ \dfrac{\int_\Omega \int_\Omega  I_\rho(x-y) u_i(x) (f(x)-b^{k-1}(y)C_i^k)^2 \ dydx   }{\int_\Omega \int_\Omega  I_\rho(x-y) u_i(y)  \ dydx }},\\
%b^k(x) = \dfrac{\sum_{i=1}^n [C_i^k/(\nu_i^k)^2] I_\rho*(fu_i^k)}{\sum_{i=1}^n [(C_i/\nu_i^k)^2]I_\rho*(u_i^k)} .
%\end{cases}
%\end{equation}
We then evaluate $\phi^k$  according to  Step 2 in Algorithm~\ref{a:MBO},  which is then followed by the thresholding step ({\it i.e.} Step 3) to determine $u^{k+1}$. 

We now show numerical examples and compare our results with those in 	Zhang et al. \cite{zhang2013local} where level-set approach is used. In this numerical computation, we use the image  domain $\Omega=[-\pi,\pi]^2$.  The convolutions are efficiently evaluated by FFT.

\subsubsection{A star-shaped image with intensity inhomogeneity} 
We start from a classical star-shaped image with ground-truth. Figure~\ref{fig:1} shows the segmentation results for five images with different levels of intensity inhomogeneity.  The table in Figure~\ref{fig:1} shows the efficiency and robustness of the proposed ICTM when compared with the level-set method \cite{zhang2013local}. The number of iterations needed for the ICTM to converge remains almost the same at $7$ for different intensity inhomogeneity, while the number of iterations  increases from $7$ to about $240$ for the level-set method in Zhang et al. \cite{zhang2013local}. 
%The energy plot verifies the energy decaying property of the ICTM as we proved in Theorem~\ref{thm:stability} and \ref{thm:stability2}. We note that the energy may have negative values because $\nu_i^k<1$ at some iterations.
We also use the Jaccard similarity (JS) as an index to measure the accuracy of our segmentation. The JS index between two regions $S_1$ and $S_2$ is calculated as $JS(S_1,S_2) = |S_1 \cap S_2|/|S_1 \cup S_2|$, which describes the ratio between the intersection areas of $S_1$ and $S_2$. In the five experiments in Figure~\ref{fig:1}, we have $JS(S_1, S_2) = 1, 1, 0.9997, 0.9985,$ and $0.9985$, respectively, when we set $S_1$ as the numerical result and $S_2$ as the ground truth. The parameters in the five experiments are all fixed as $(\rho, \gamma, \tau) = (15, 0.1, 0.001)$. \begin{figure*}[ht]
\centering
\includegraphics[width=0.19\textwidth,clip,trim= 6cm 6cm 6cm 6cm]{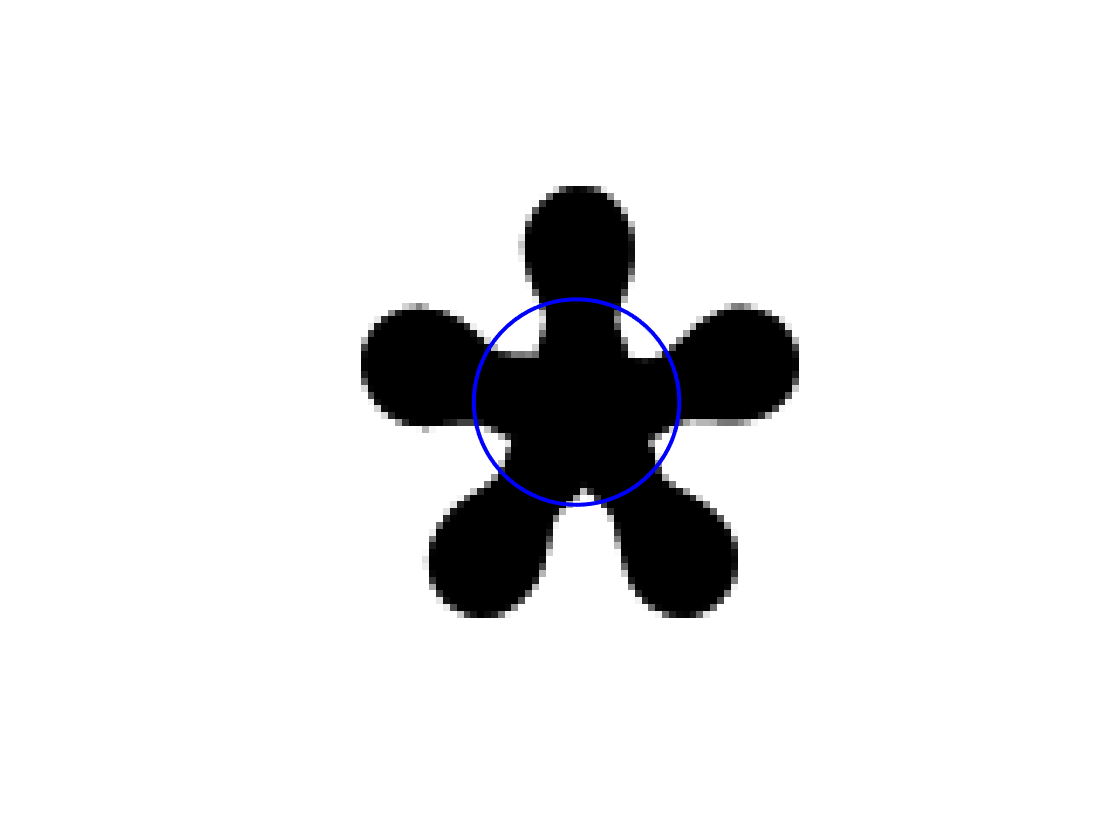}
\includegraphics[width=0.19\textwidth,clip,trim= 6cm 6cm 6cm 6cm]{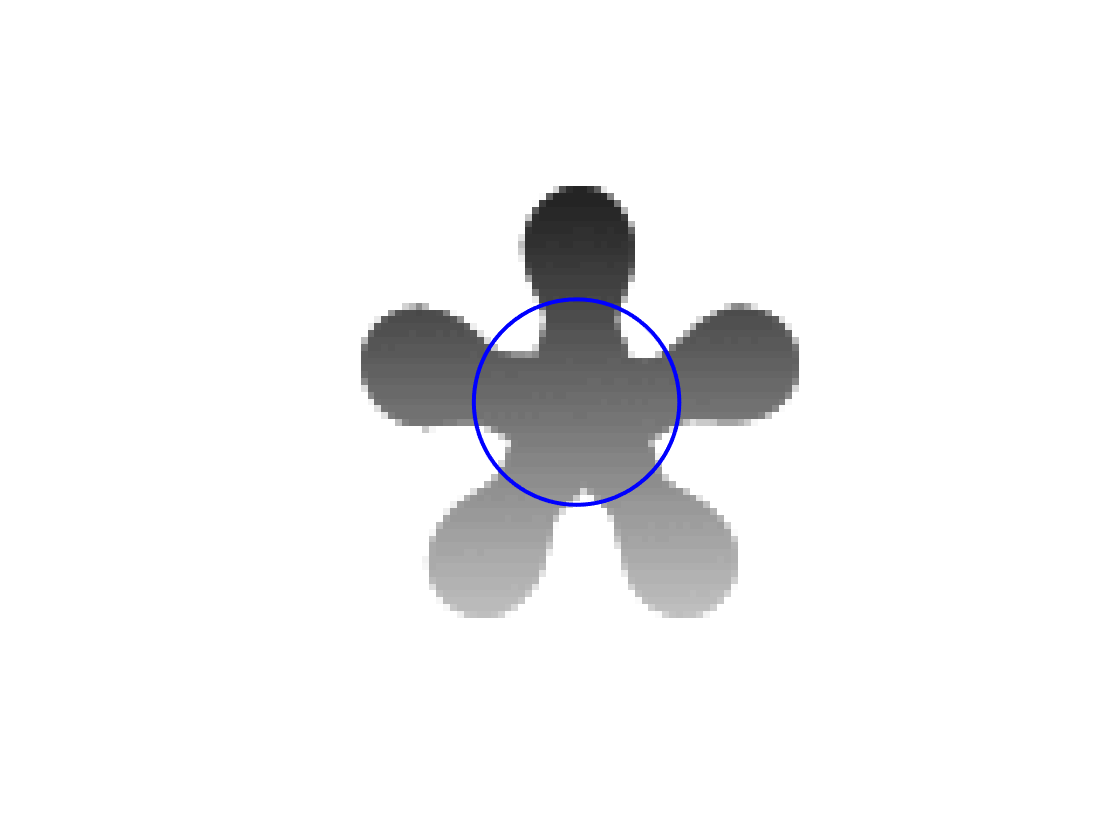}
\includegraphics[width=0.19\textwidth,clip,trim= 6cm 6cm 6cm 6cm]{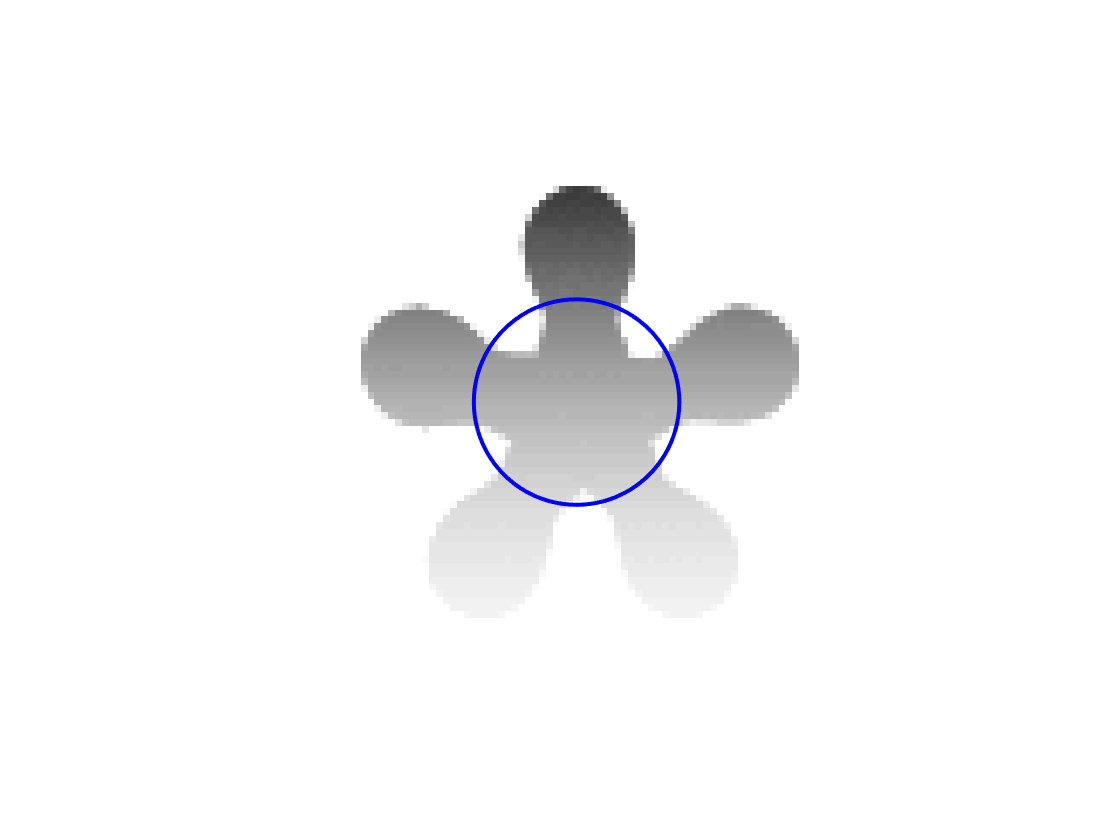}
\includegraphics[width=0.19\textwidth,clip,trim= 6cm 6cm 6cm 6cm]{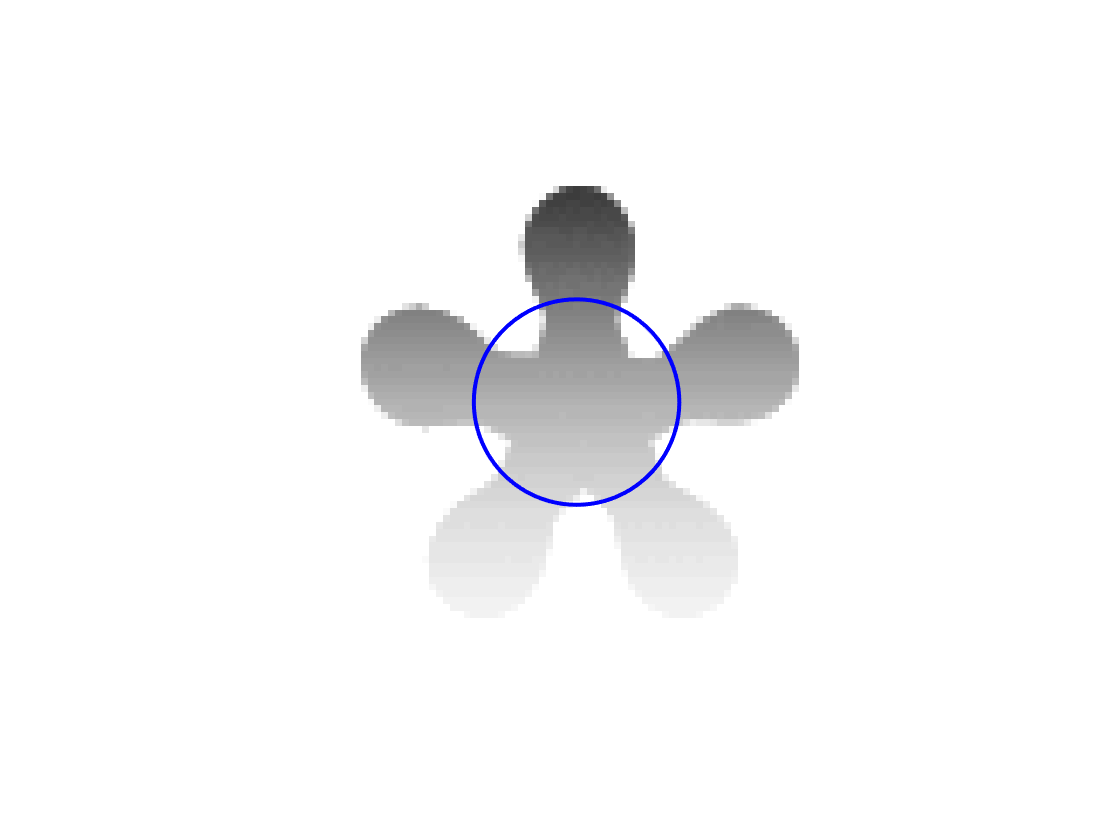}
\includegraphics[width=0.19\textwidth,clip,trim= 6cm 6cm 6cm 6cm]{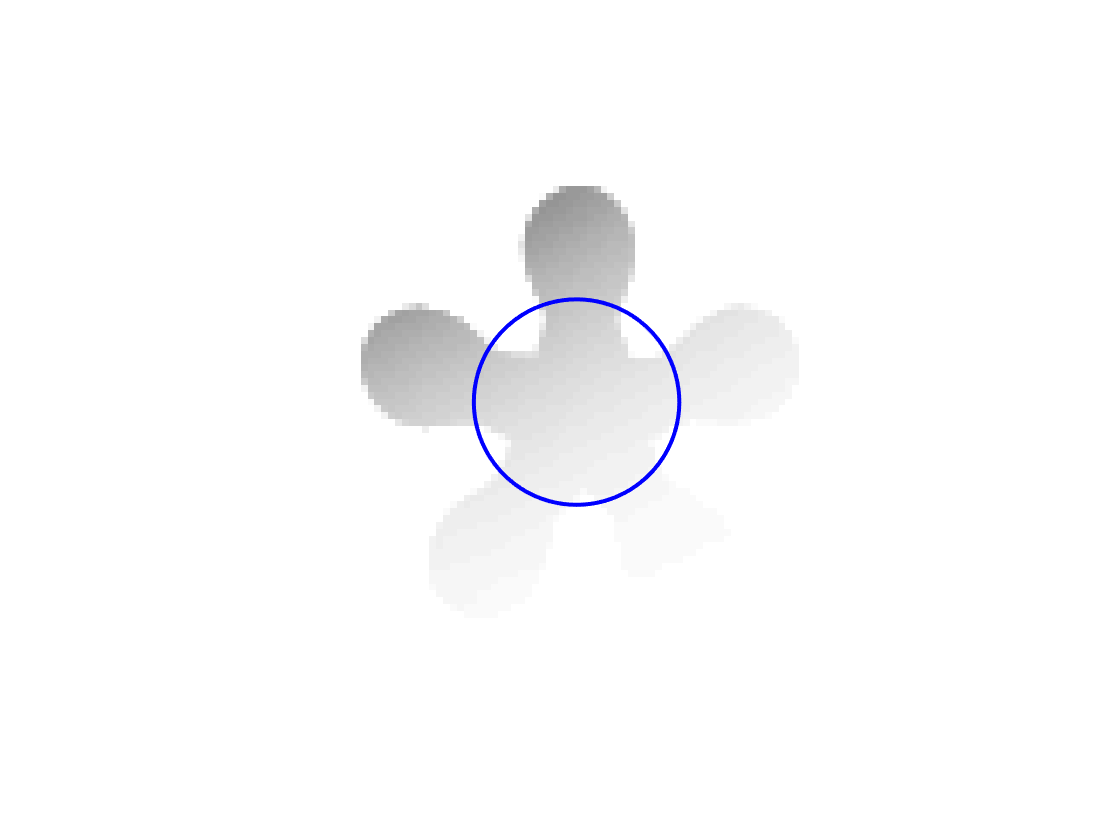}
\includegraphics[width=0.19\textwidth,clip,trim= 6cm 6cm 6cm 6cm]{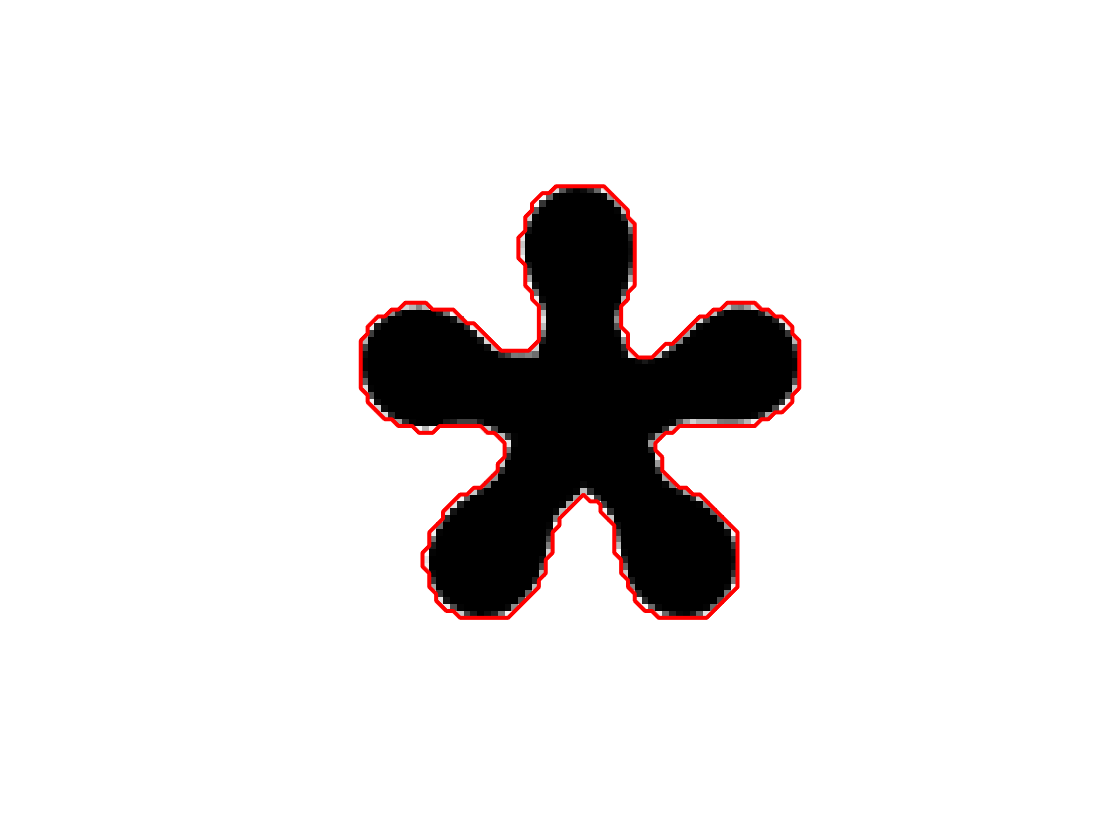}
\includegraphics[width=0.19\textwidth,clip,trim= 6cm 6cm 6cm 6cm]{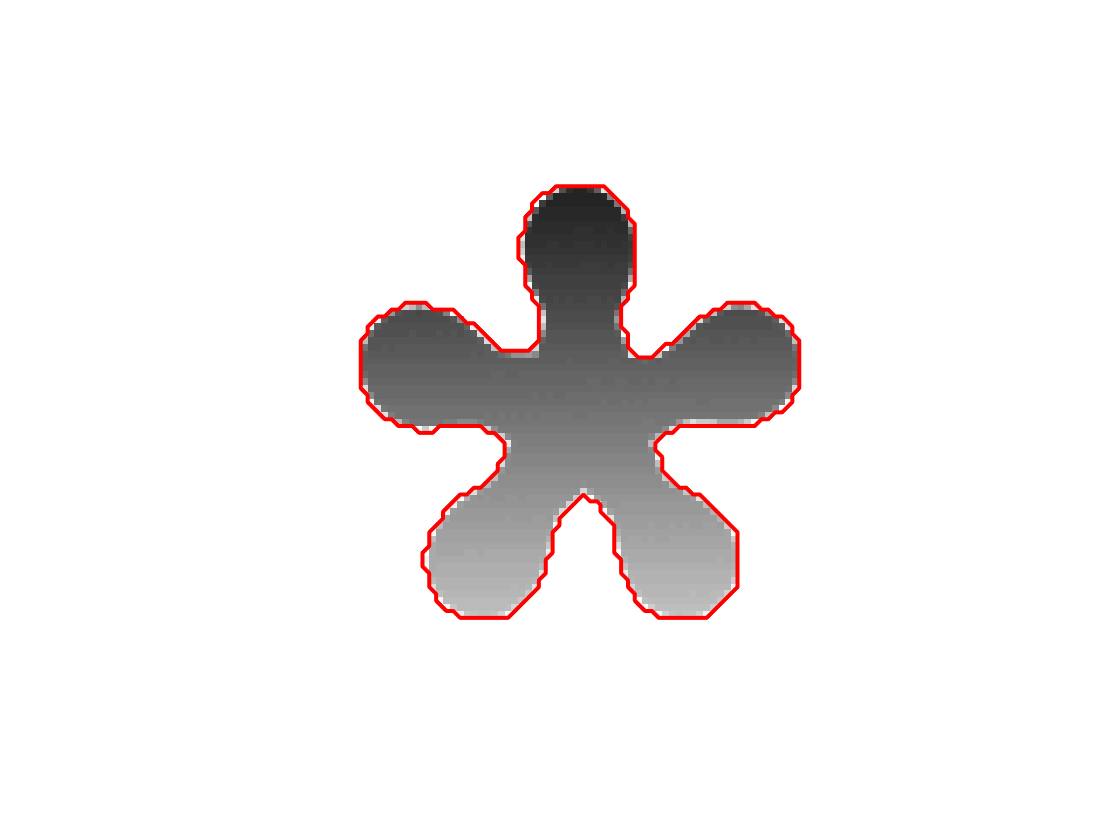}
\includegraphics[width=0.19\textwidth,clip,trim= 6cm 6cm 6cm 6cm]{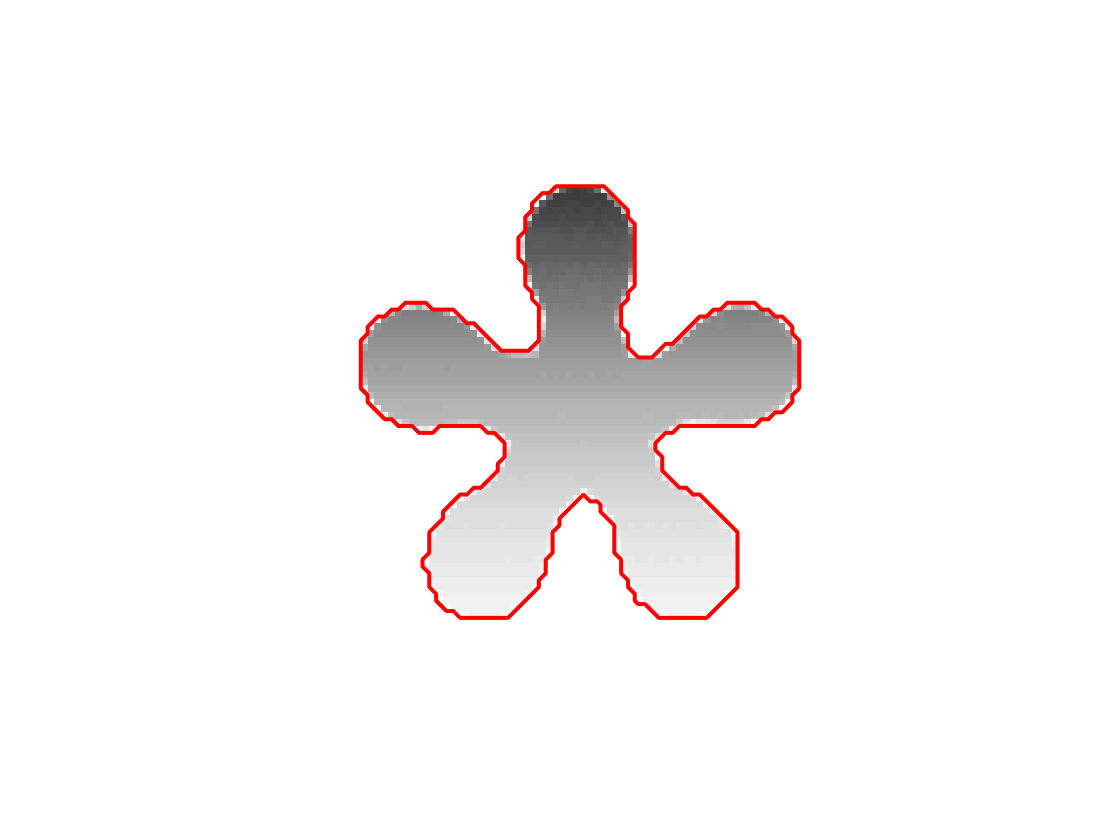}
\includegraphics[width=0.19\textwidth,clip,trim= 6cm 6cm 6cm 6cm]{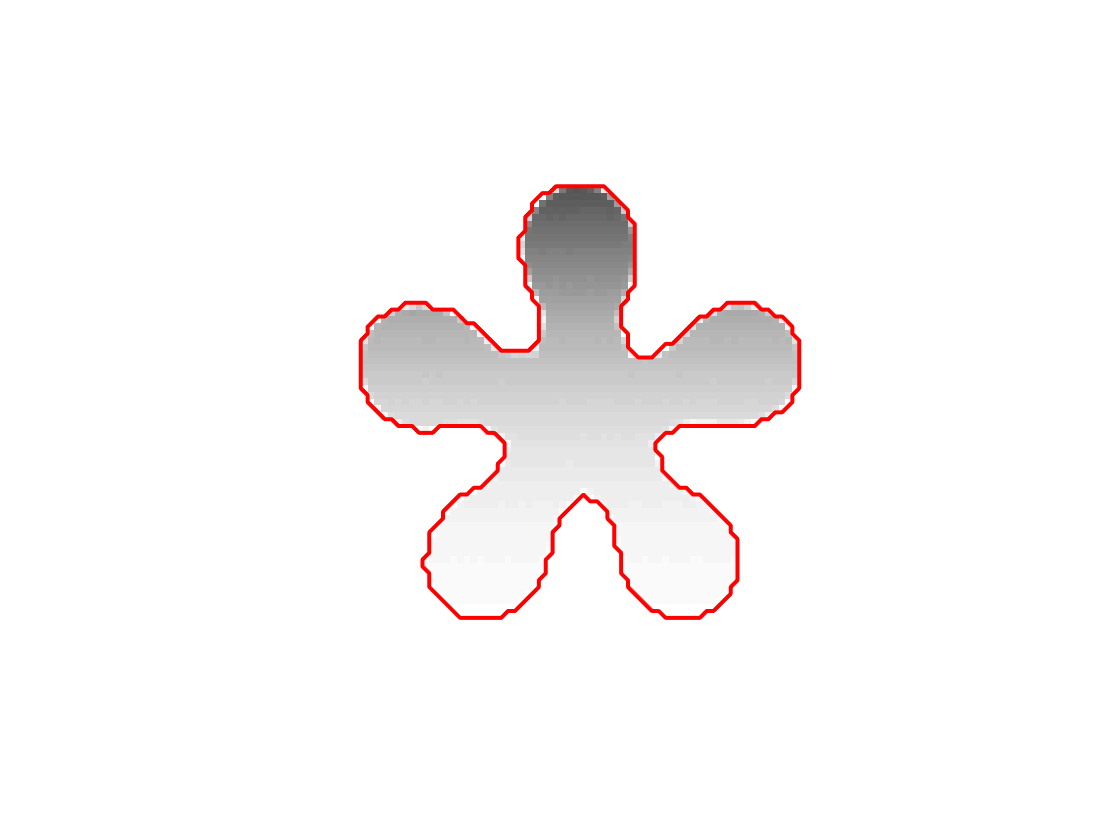}
\includegraphics[width=0.19\textwidth,clip,trim= 6cm 6cm 6cm 6cm]{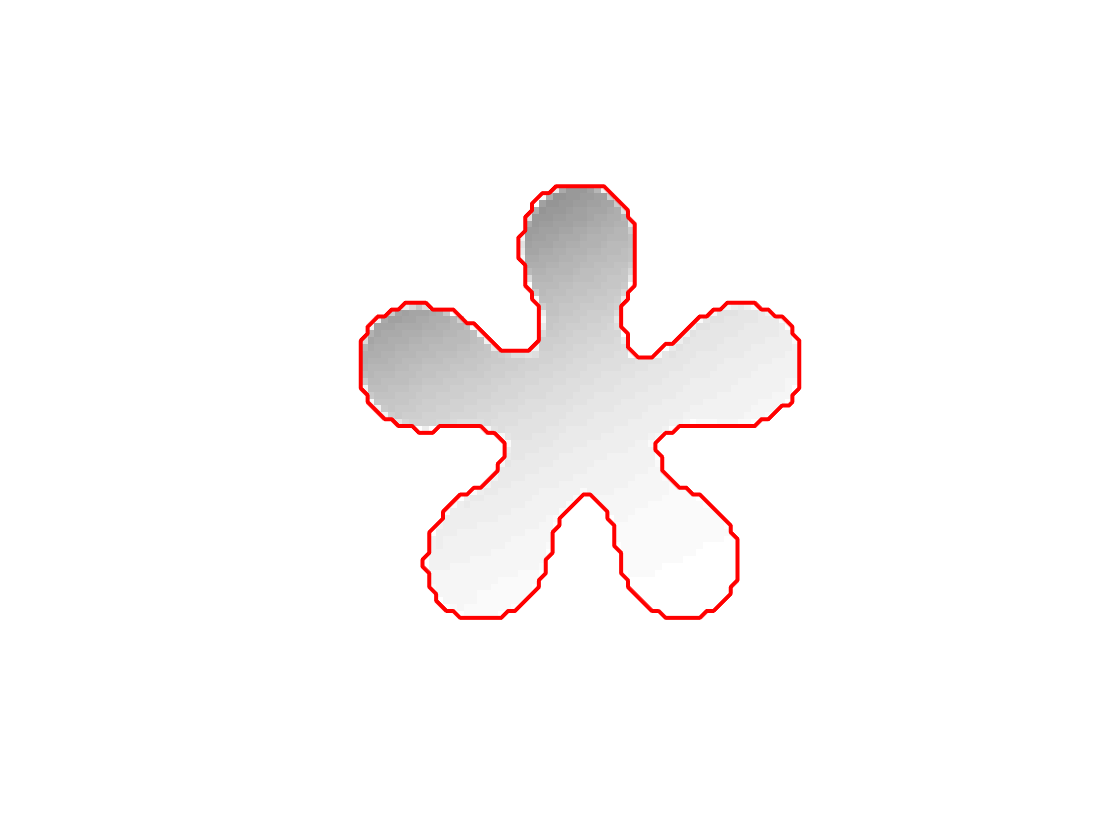}

\medskip

\begin{tabular}{|c|c|c|c|c|c|}
\hline
\# of iterations of the ICTM & 8  & 7 & 7& 7 & 7  \\ 
\hline
\# of iterations of the level-set method \cite{zhang2013local}&7 & 13 & 35& 186 & 239\\
\hline
\end{tabular}
\caption{First row: Initial contour of the same image with different intensity inhomogeneity. Second row: The segmented region. Table: Comparison of  the number of iterations for each case from left to right between the ICTM and the level-set method used in Zhang et al. \cite{zhang2013local}. In all five experiments, we set $\rho = 15$, $\gamma = 0.1$, and $\tau = 0.001$. The results for the level-set method are obtained using the software code from \url{https://www4.comp.polyu.edu.hk/~cslzhang/LSACM/LSACM.htm}. See Section~\ref{subsec:1} for details.} \label{fig:1}
\end{figure*}

\subsubsection{Noisy intensity inhomogeneity images}
We then apply the ICTM to five different, noisy  intensity-inhomogeneous  images.  The results in Figure~\ref{fig:2} again show that our  ICTM  is efficient and accurate. The parameters for the five figures from left to right are $(\rho, \gamma, \tau) = (15, 0.1, 0.02)$, $(5, 0.15, 0.03)$, $(10, 0.02, 0.01)$, $(10, 0.7, 0.03)$, and $(10, 0.035, 0.002)$. Numbers of iterations in the ICTM are 5, 30, 28, 35, and 18. However, the numbers of iterations in the level-set method are $57$, $219$, $670$, $290$, and $230$. The table in Figure~\ref{fig:2} shows that the ICTM  is an order of magnitude faster than the level-set method.  
%Also, in each iteration, we only need to evaluate the convolutions and there is no need to do regularization and reinitialization as in  the level set function.
 
\begin{figure*}[ht]
\centering
\includegraphics[width=0.19\textwidth,clip,trim= 6cm 2cm 6cm 2cm]{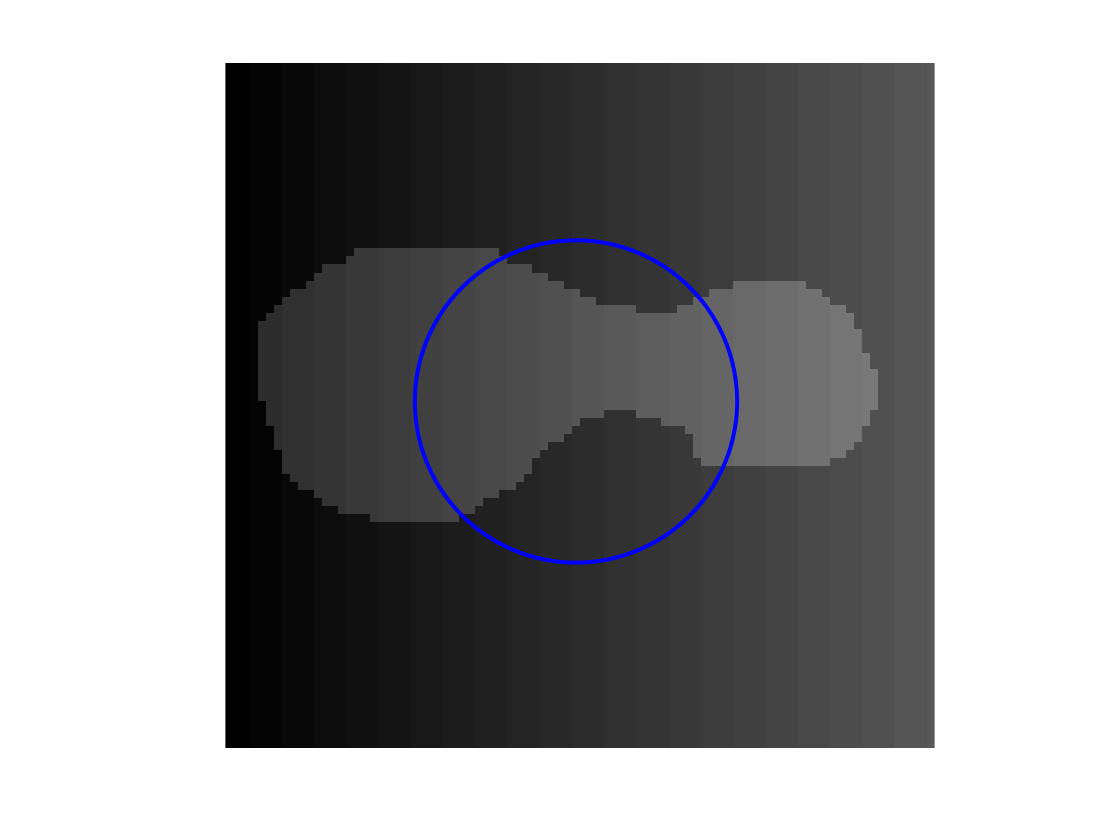}
\includegraphics[width=0.19\textwidth,clip,trim= 6cm 2cm 6cm 2cm]{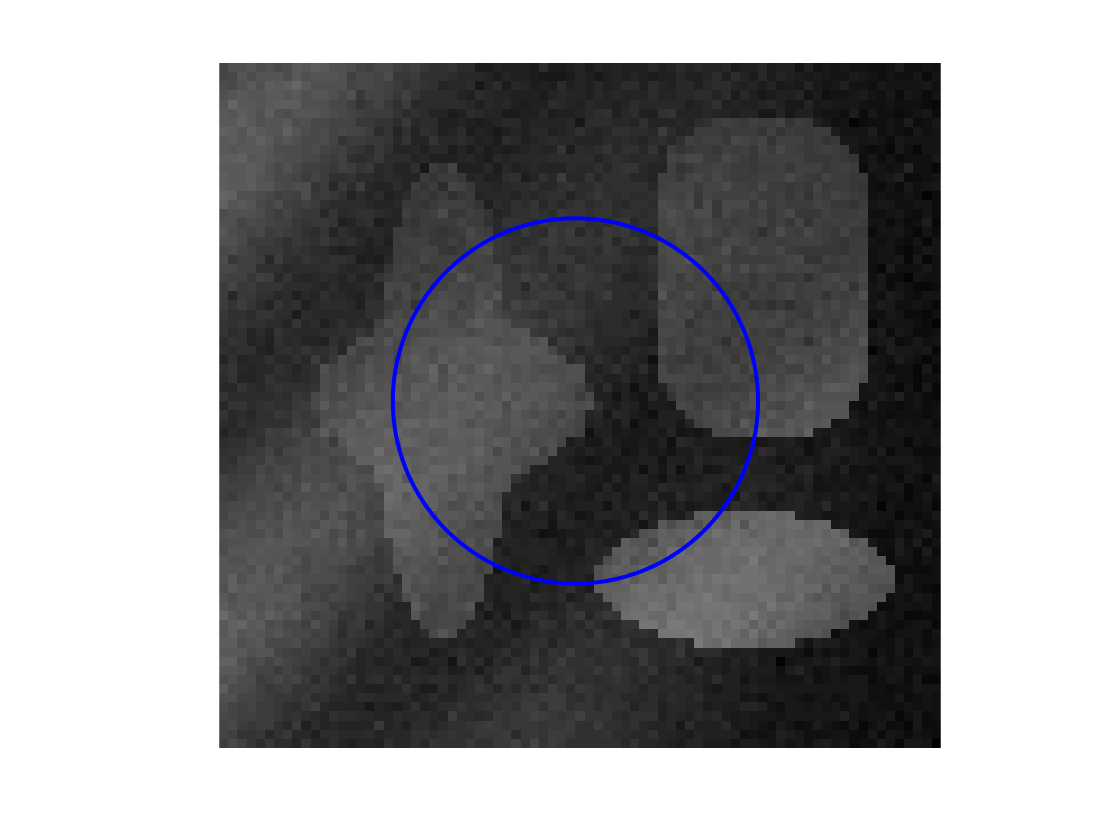}
\includegraphics[width=0.19\textwidth,clip,trim= 6cm 2cm 6cm 2cm]{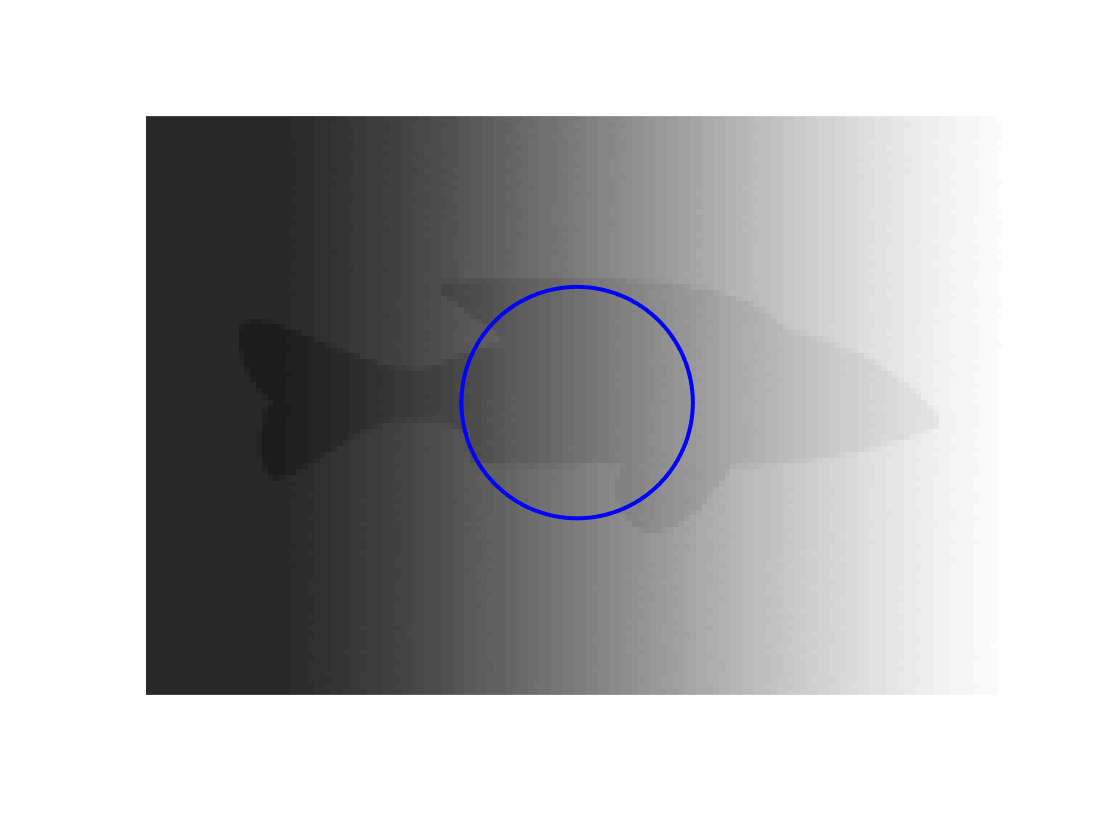}
\includegraphics[width=0.19\textwidth,clip,trim= 6cm 2cm 6cm 2cm]{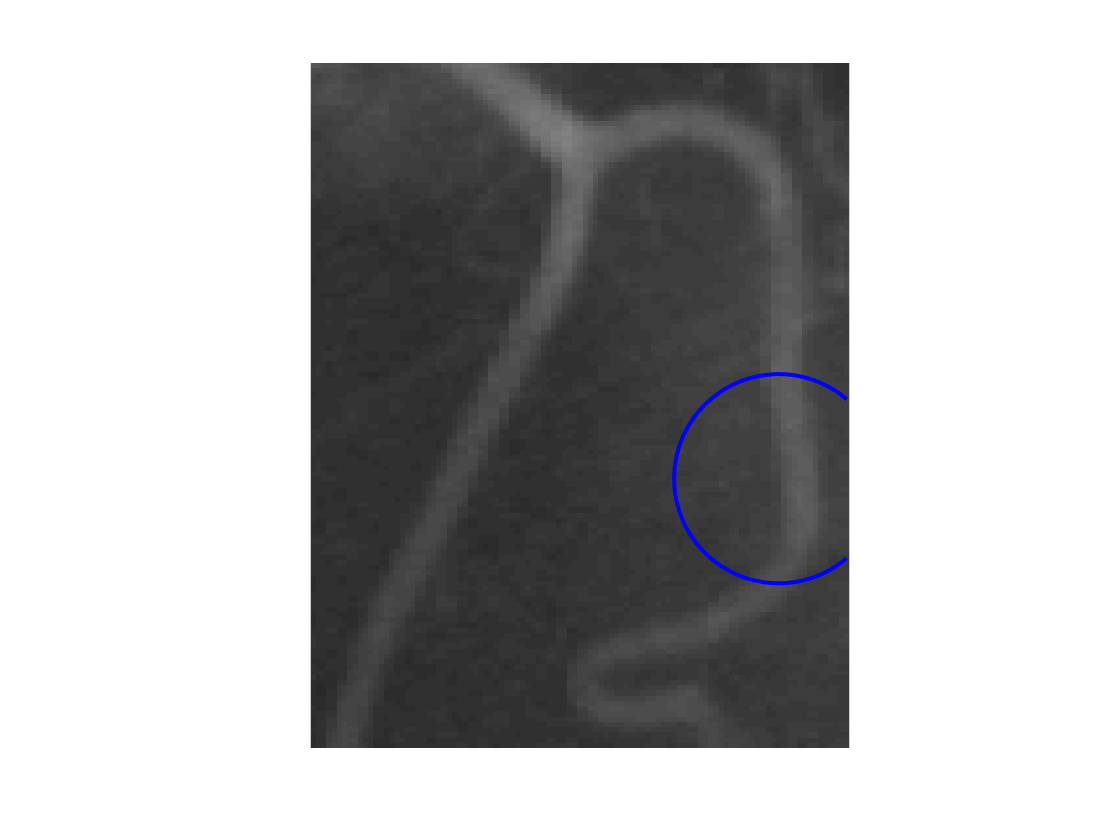}
\includegraphics[width=0.19\textwidth,clip,trim= 6cm 2cm 6cm 2cm]{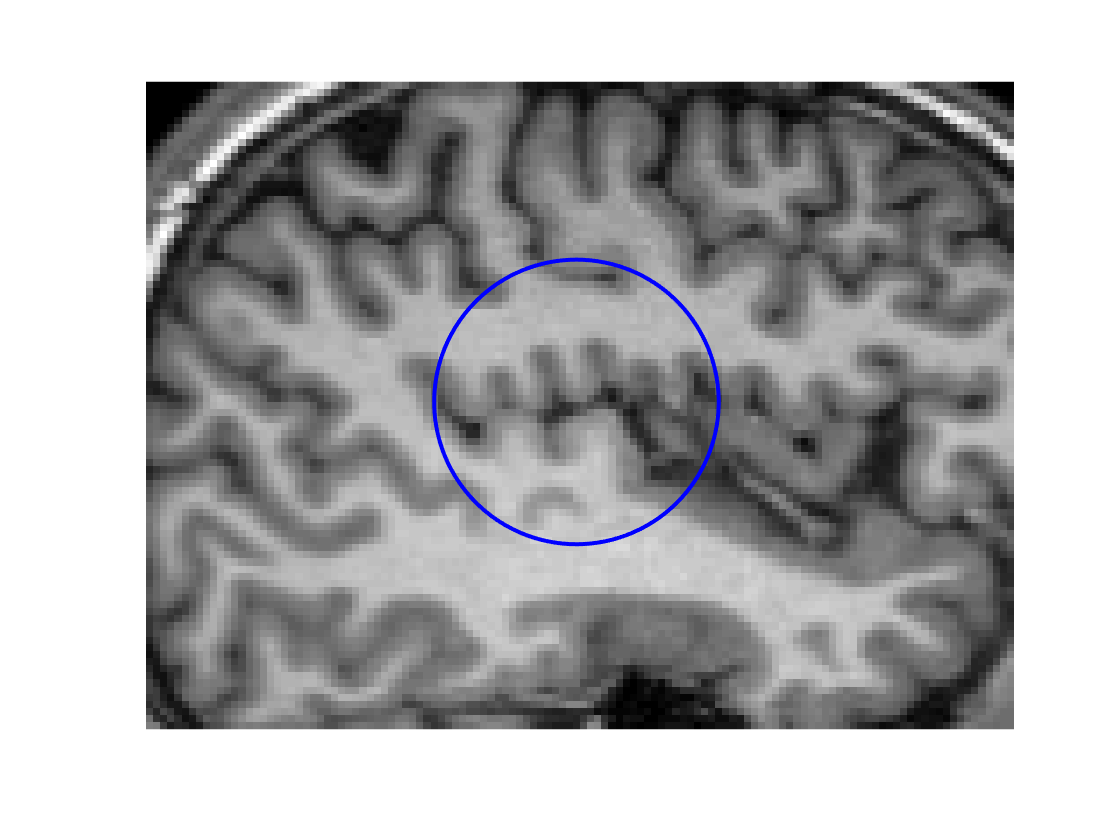}\\
\includegraphics[width=0.19\textwidth,clip,trim= 6cm 2cm 6cm 2cm]{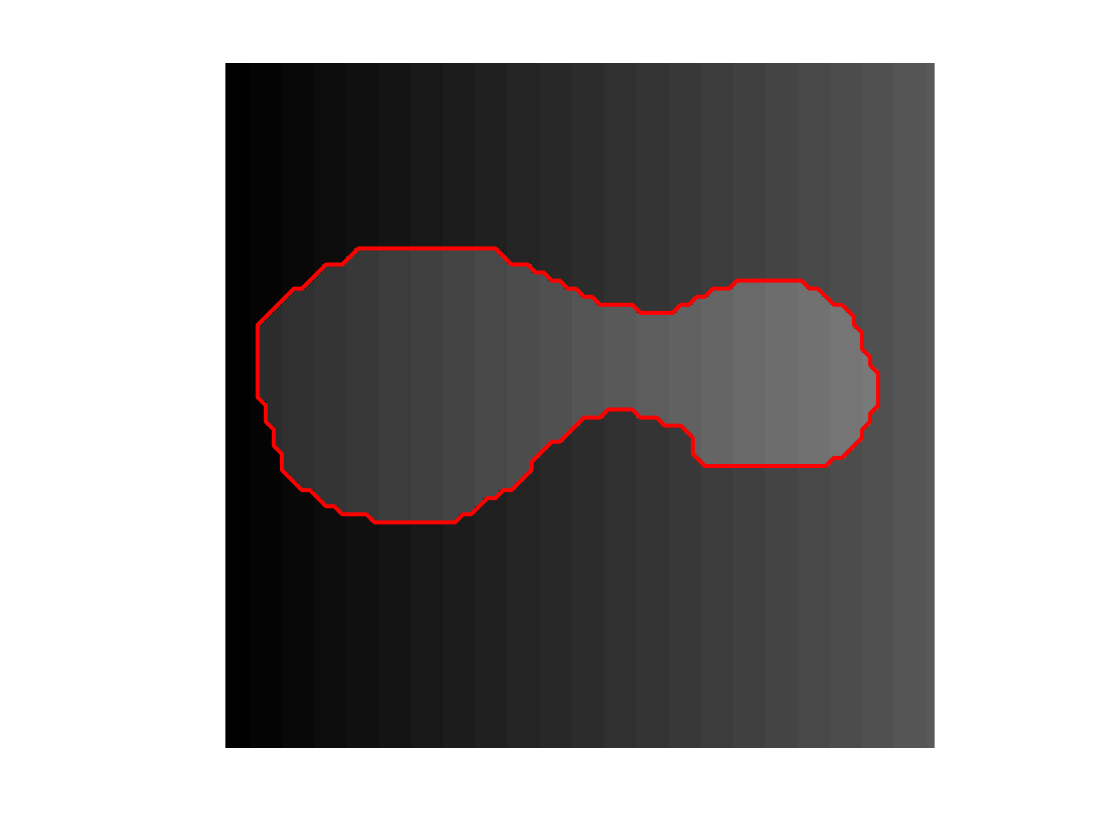}
\includegraphics[width=0.19\textwidth,clip,trim= 6cm 2cm 6cm 2cm]{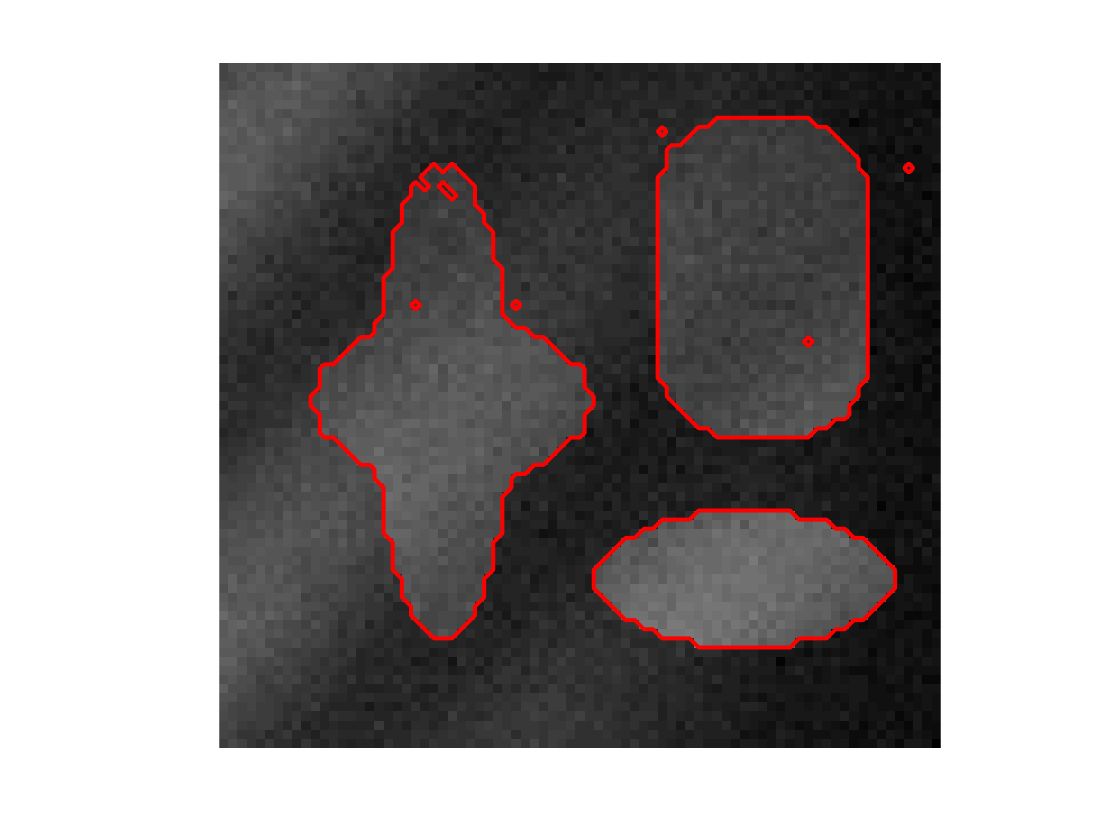}
\includegraphics[width=0.19\textwidth,clip,trim= 6cm 2cm 6cm 2cm]{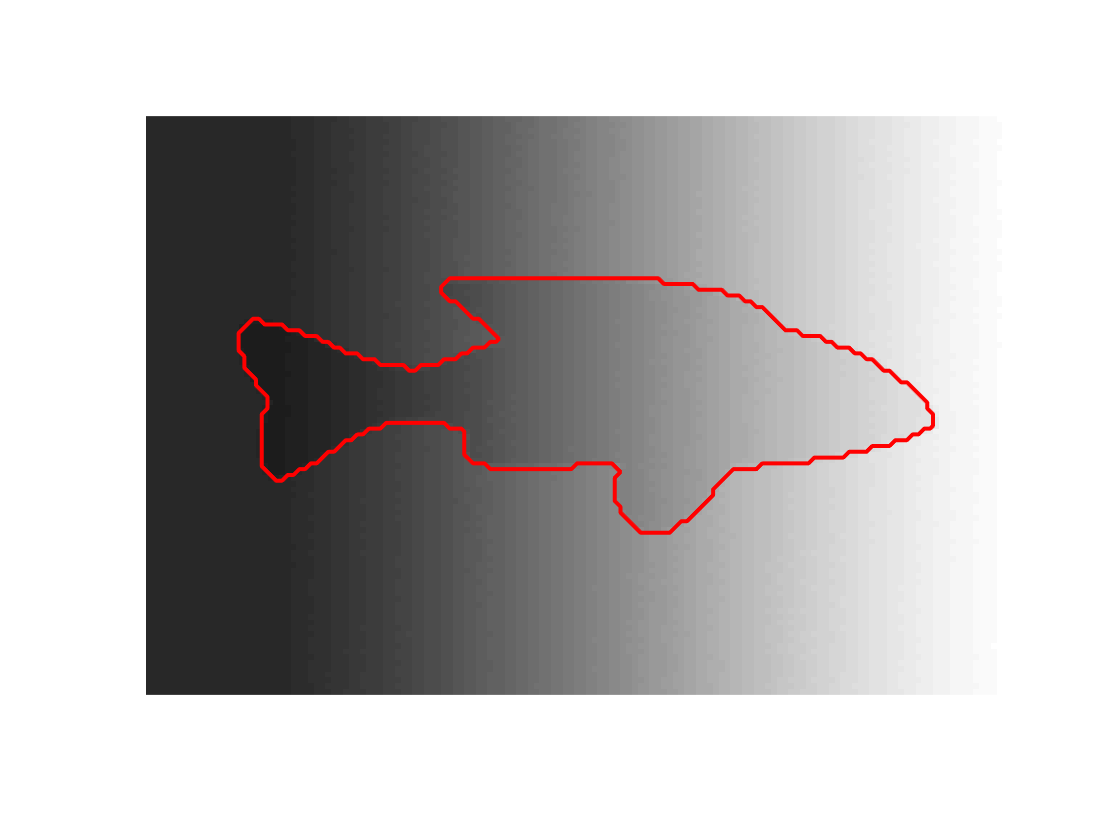}
\includegraphics[width=0.19\textwidth,clip,trim= 6cm 2cm 6cm 2cm]{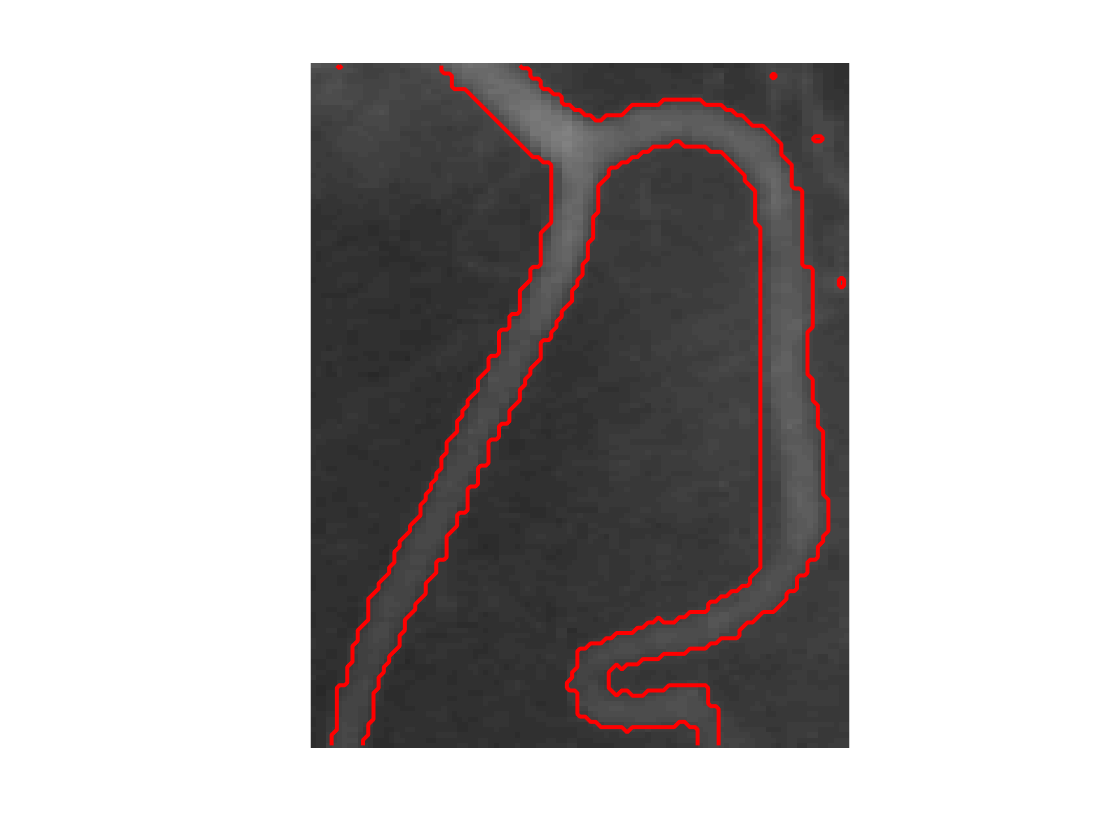}
\includegraphics[width=0.19\textwidth,clip,trim= 6cm 2cm 6cm 2cm]{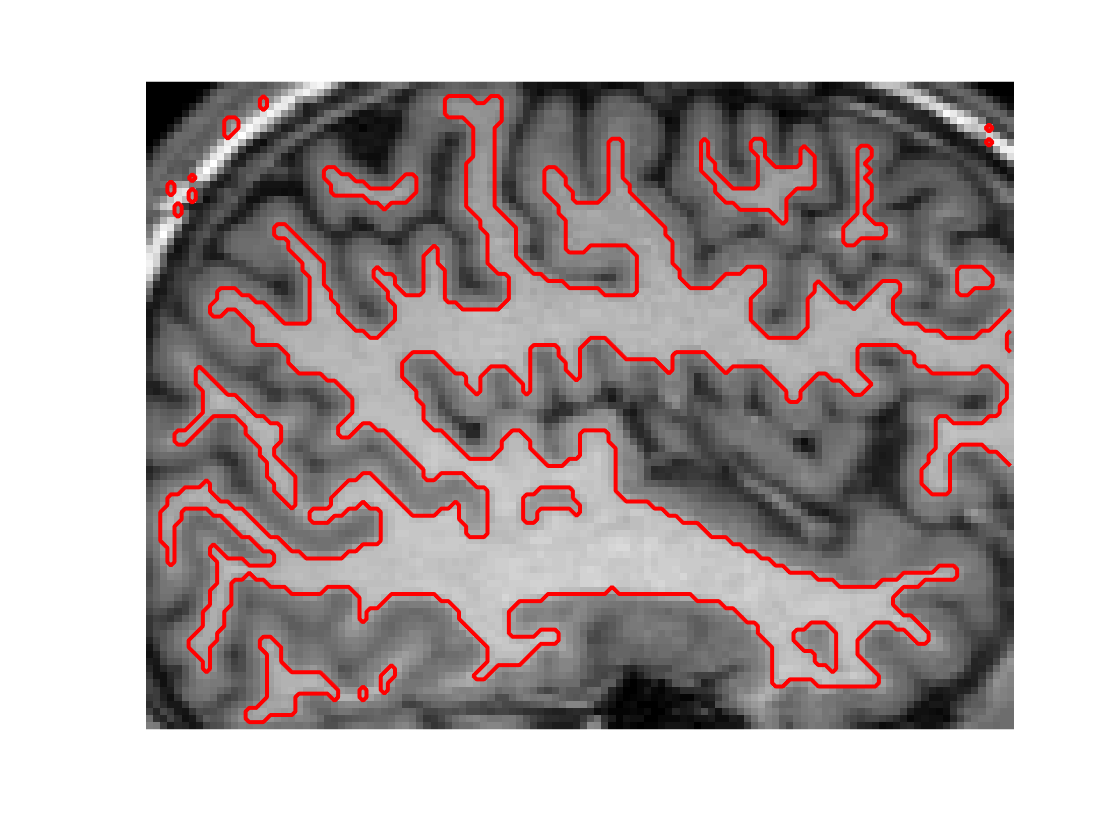}

\medskip

\begin{tabular}{|c|c|c|c|c|c|}
\hline
\# of iterations of the ICTM & 5  & 30 & 28 & 35 & 18  \\ 
\hline
\# of iterations of the level-set method \cite{zhang2013local}&  57 & 219 & 670 & 290 & 230\\
\hline
\end{tabular}

\caption{Initial contour and segmented region using the ICTM in the LSAC model. The parameters from left to right are $(\rho, \gamma, \tau) = (15, 0.1, 0.02)$, $(5, 0.15, 0.03)$, $(10, 0.02, 0.01)$, $(10, 0.7, 0.03)$, and $(10, 0.035, 0.002)$. The results for the level-set method are obtained using the software code from \url{https://www4.comp.polyu.edu.hk/~cslzhang/LSACM/LSACM.htm}.  See Section~\ref{subsec:1} for details.} \label{fig:2}
\end{figure*}

\subsection{Applications to the Local Intensity Fitting (LIF) model  \eqref{LBF}}\label{subsec:2}
Finally, we apply the ICTM to the LIF model \eqref{LBF} for the two-phase case. In this case, we choose $F_i(f,\Theta_1,\Theta_2, \ldots, \Theta_n) = \mu_i \int_{\Omega}G_{\sigma}(x-y)|C_i(x)-f(y)|^2 \ dx$ and $\Theta_i =  C_i(x)$ for any $i\in[2]$. When $(u_1^k, \ldots, u_n^k)$ are fixed,
$$\mathcal{E}_f = \int_{\Omega} u^k F_1(f,C_1,C_2)+(1-u^k) F_2(f,C_1, C_2) \ dy$$ is strictly convex with respect to $C_i(x)$, $i\in[2]$. 
Then, direct calculations reduce Step 1 in Algorithm~\ref{a:MBO} to 
\begin{align*} 
&\iint_{\Omega} u^k(y) G_\sigma (x-y) [C_1(x)-f(y)] \ dydx= 0,  \\    
&\iint_{\Omega} (1-u^k(y)) G_\sigma (x-y) [C_2(x)-f(y)] \ dydx= 0 \\
 \end{align*}
whose solutions are given by
\begin{equation} \label{lbfC}
C_1^k(x)  = \frac{G_{\sigma}*( u^k f)}{G_{\sigma}* u^k}, \  \ C_2^k(x)  = \frac{G_{\sigma}*( (1-u^k) f)}{G_{\sigma}*(1-u^k)}.
\end{equation}
%In Algorithm~\ref{a:MBO}, the solutions of $C_1$ and $C_2$ are
%\begin{equation}
%C_1^k(x)  = \frac{G_{\sigma}*( u^k f)}{G_{\sigma}* u^k}, \ \ 
%C_2^k(x)  = \frac{G_{\sigma}*( (1-u^k) f)}{G_{\sigma}* (1-u^k)}.
%\end{equation}

\begin{rem}
In \eqref{lbfC}, $C_i^k(x)$ may not be defined at some $x\in \Omega$ since $G_{\sigma}* u^k$ or $G_{\sigma}* (1-u^k)$can be zero (at least numerically). Since $G_{\sigma}* u^k \geq 0$ and $G_{\sigma}*(1-u^k) \geq 0$, we add a small number  $\varepsilon>0$ in both the numerator and the denominator as follows, 
\begin{align*}
&C_1^k(x)  = \frac{G_{\sigma}*( u^k f)+\varepsilon}{G_{\sigma}* u^k+\varepsilon}, \\ & C_2^k(x)  = \frac{G_{\sigma}*((1-u^k) f)+\varepsilon}{G_{\sigma}* (1-u^k)+\varepsilon}.
\end{align*}
In the subsequent examples, we set $\varepsilon = 10^{-6}$.
\end{rem}
Again, the evaluation of $\phi^k$ in Step 2 of  Algorithm~\ref{a:MBO} from \eqref{solutionTheta1} is followed by the thresholding step (i.e. Step 3) to determine $u^{k+1}$. 

We now show numerical examples and compare our results with those in Li et al. \cite{Chunming_Li_2008} using  the level set method.
To be consistent with  the code of \cite{Chunming_Li_2008} from \url{http://www.imagecomputing.org/~cmli/code/}, we use the two-dimensional Gaussian low-pass filter instead of the Gaussian kernel $G_\sigma$ to avoid specifying the domain size of $\Omega$. The filter can be generated by the MATLAB's \texttt{fspecial} function. 
Figure~\ref{fig:3} displays several numerical experiments on different intensity-inhomogeneous images. In all five experiments, we set $\mu_1 = \mu_2 = 1$.  In Figure~\ref{fig:3}, from left to right, we set $(\sigma, \tau, \lambda) = (20,15, 500)$, $(3, 5, 150)$, $(3, 3, 245)$, $(3,10,110)$, and $(3,2,90)$. In the table in  Figure~\ref{fig:3}, we compare the ICTM and the level-set method in Li et al. \cite{Chunming_Li_2008} in terms of the number of iterations for convergence.  In the first example from the left, the method in Li et al. \cite{Chunming_Li_2008} does not even converge.  In all other examples,  ICTM converges in significantly fewer iterations, demonstrating its very high efficiency.
% are $15$, $25$, $43$, $28$, and $47$, respectively. We use the level set code of \cite{Chunming_Li_2008} to run the same experiments. In the first example, it is very difficult to converge to the stationary solution we obtained in Figure~\ref{fig:3}.  However, our result supports the success of the model proposed in \cite{Chunming_Li_2008} in another way. 
%For the other four examples, the number of iteration steps in the level set method are $256$, $131$, $117$, and $209$, respectively. The table in Figure~\ref{fig:3} implies that the ICTM needs many fewer steps of iteration to reach the stationary state. 
%Again, in each iteration, we only need to evaluate the convolutions and there is need to do regularization and reinitialization of the level set function. Thus,  ICTM is much more efficient than level set method for the LIF model.

\begin{figure*}[ht]
\centering
\includegraphics[width=0.19\textwidth,clip,trim= 6cm 2cm 6cm 2cm]{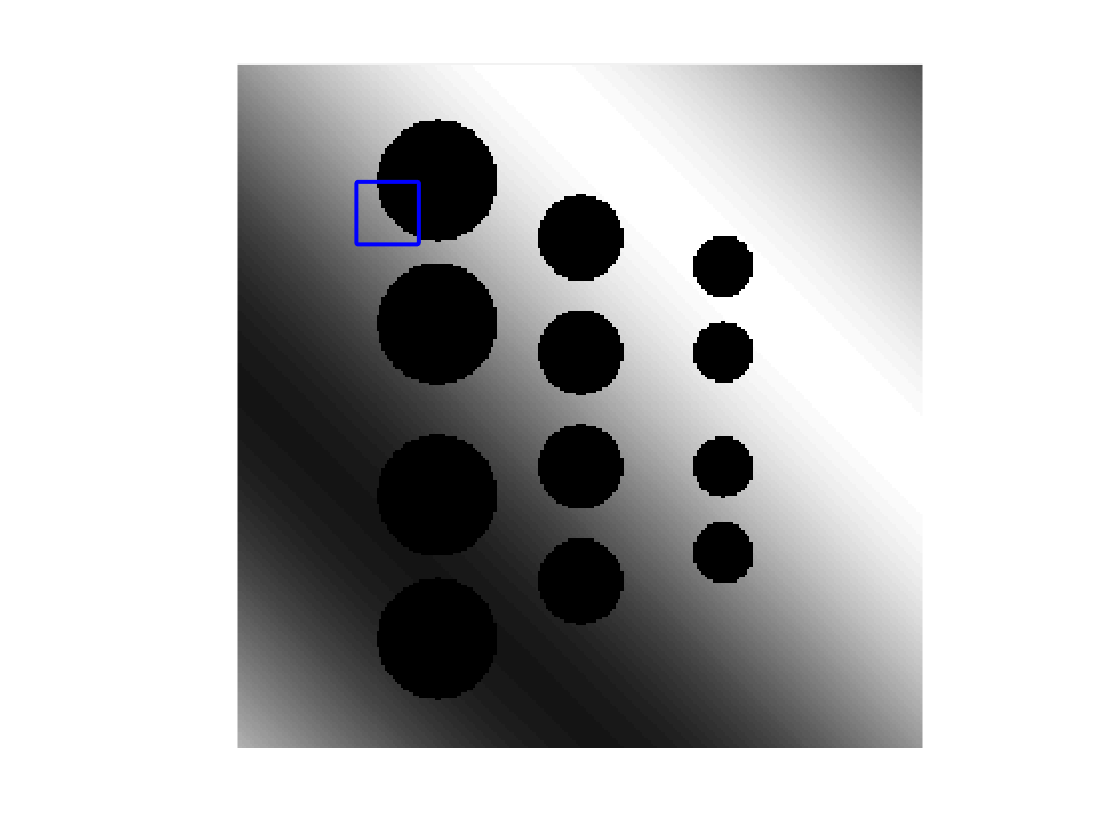}
\includegraphics[width=0.19\textwidth,clip,trim= 6cm 2cm 6cm 2cm]{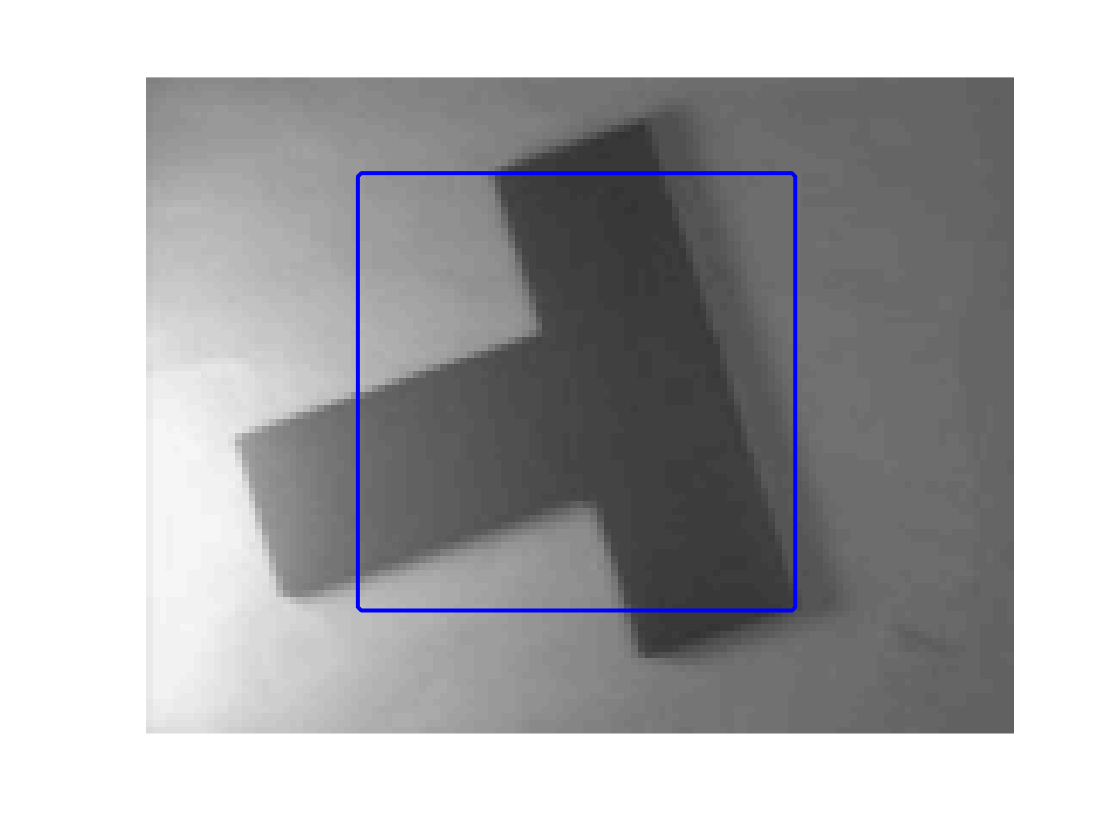}
\includegraphics[width=0.19\textwidth,clip,trim= 6cm 2cm 6cm 2cm]{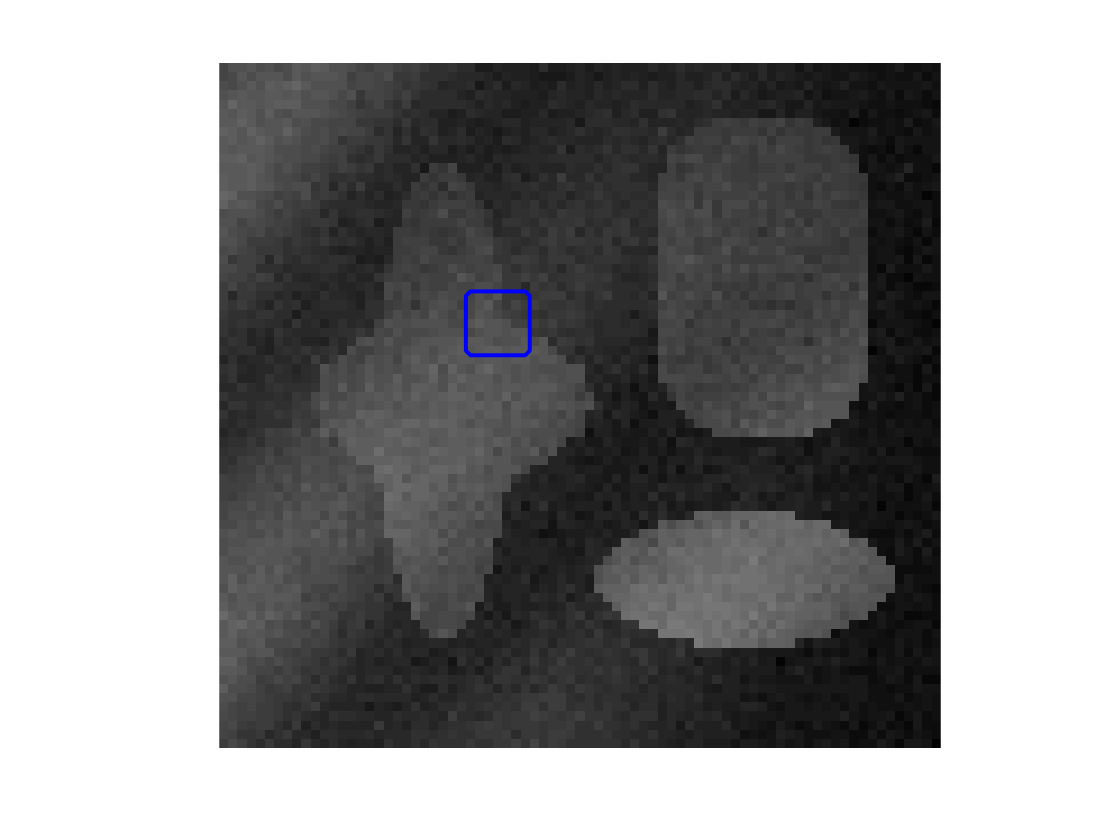}
\includegraphics[width=0.19\textwidth,clip,trim= 6cm 2cm 6cm 2cm]{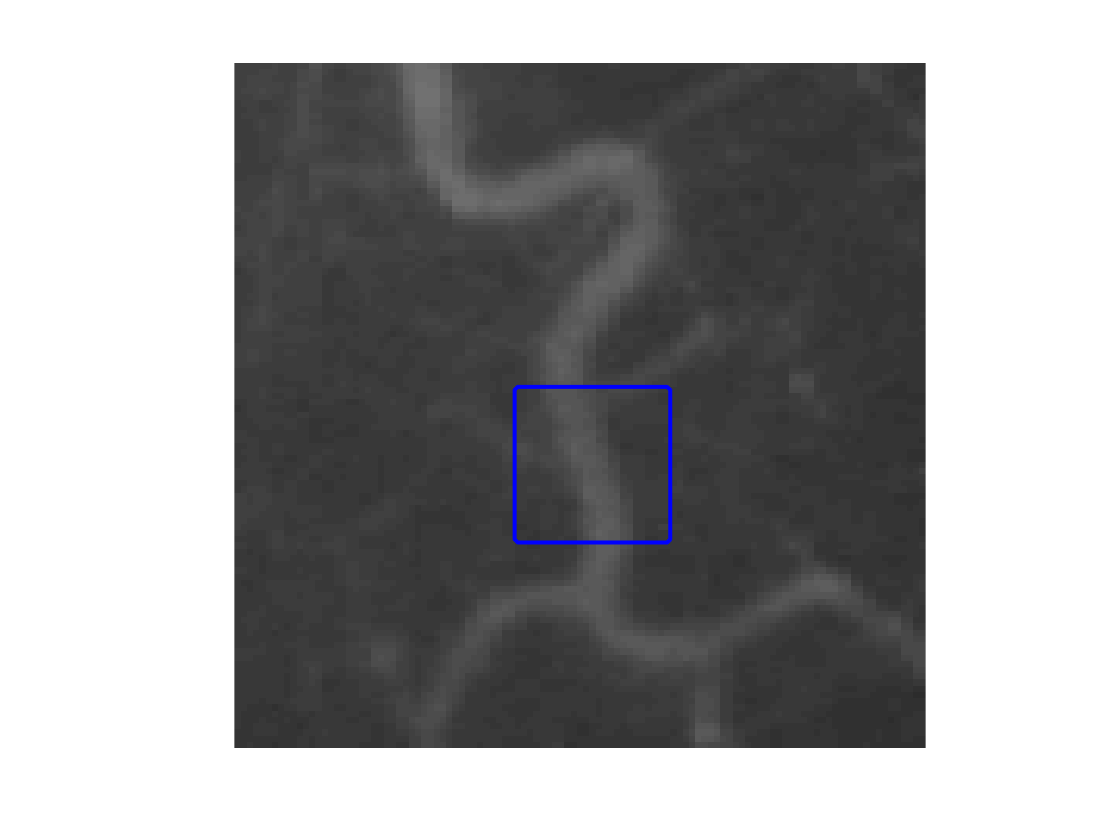}
\includegraphics[width=0.19\textwidth,clip,trim= 6cm 2cm 6cm 2cm]{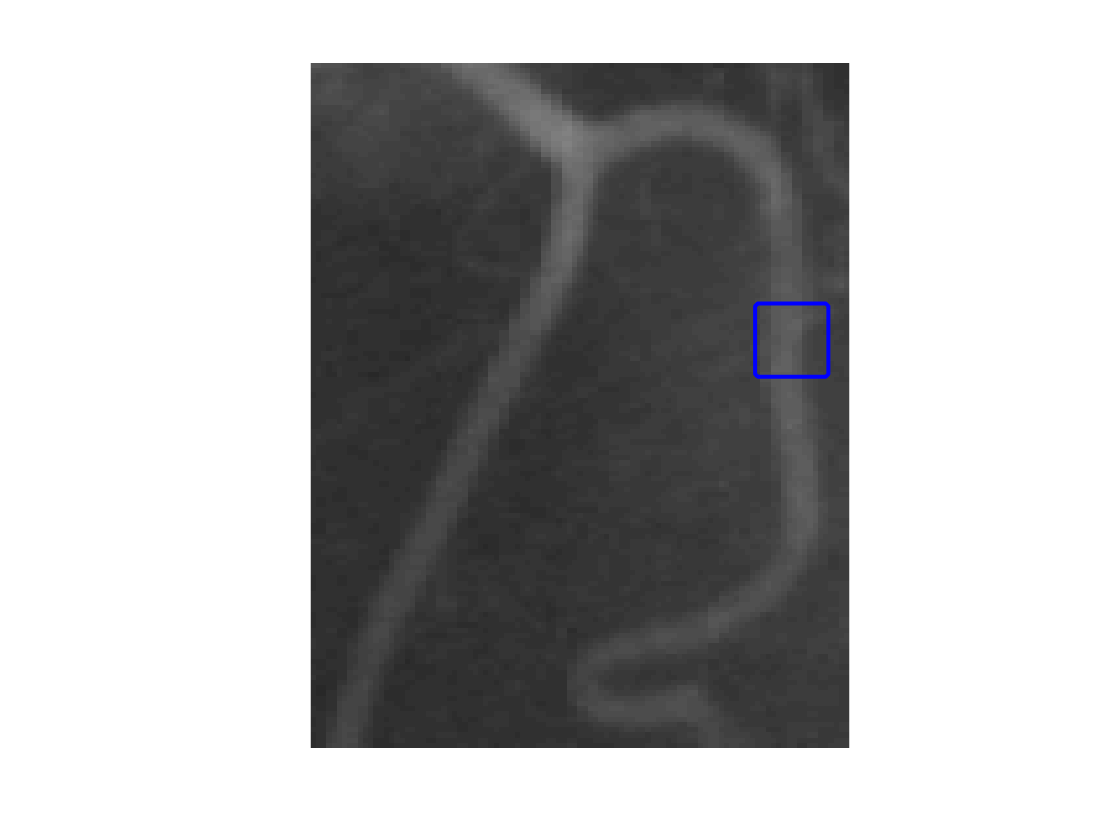}\\
\includegraphics[width=0.19\textwidth,clip,trim= 6cm 2cm 6cm 2cm]{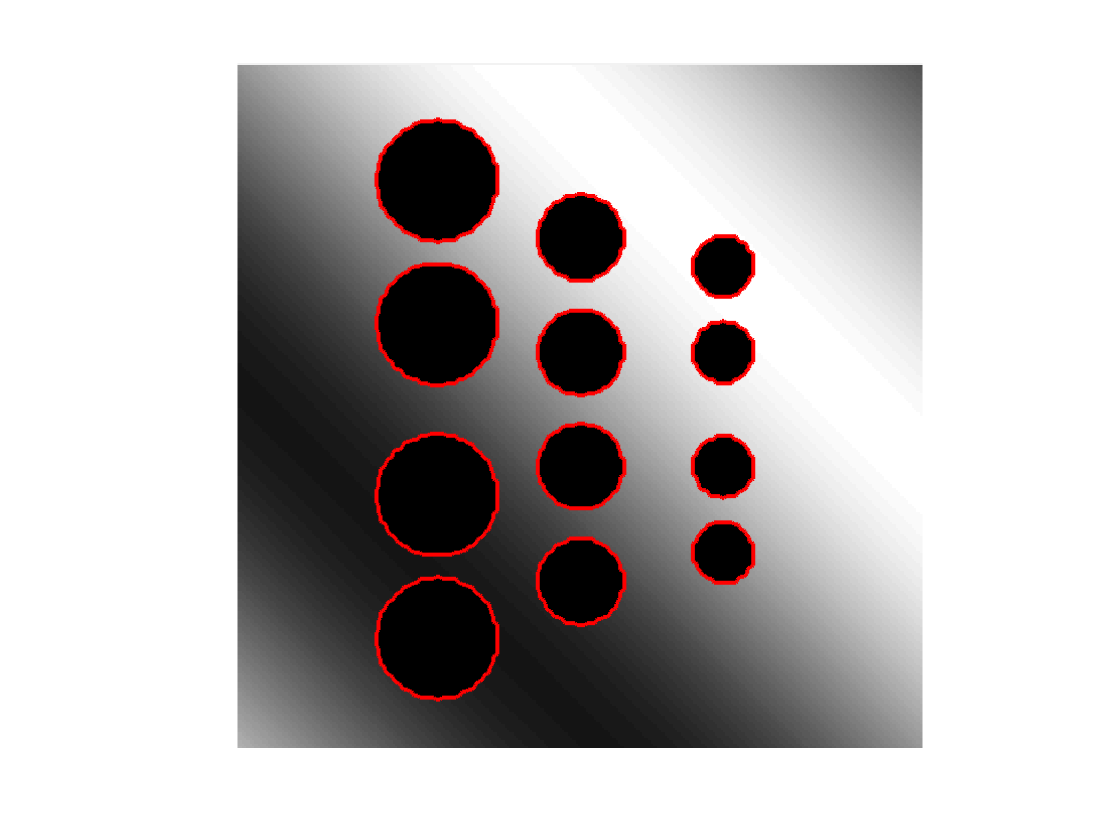}
\includegraphics[width=0.19\textwidth,clip,trim= 6cm 2cm 6cm 2cm]{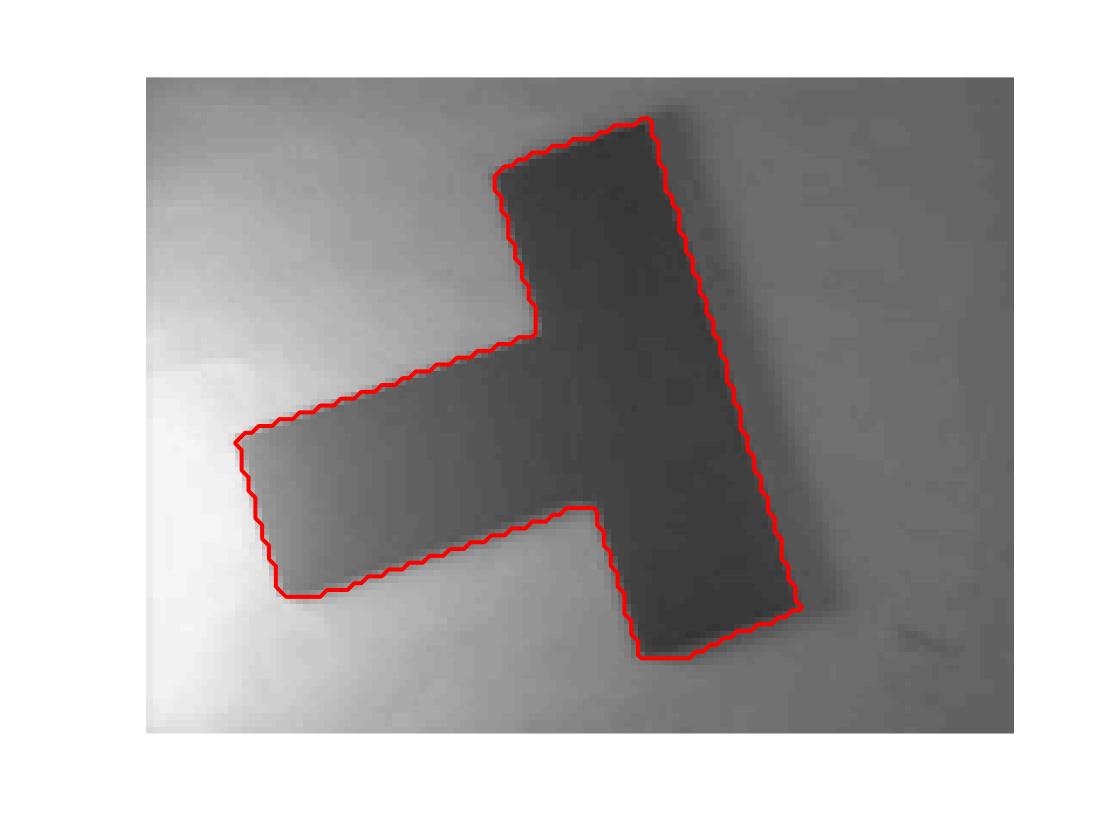}
\includegraphics[width=0.19\textwidth,clip,trim= 6cm 2cm 6cm 2cm]{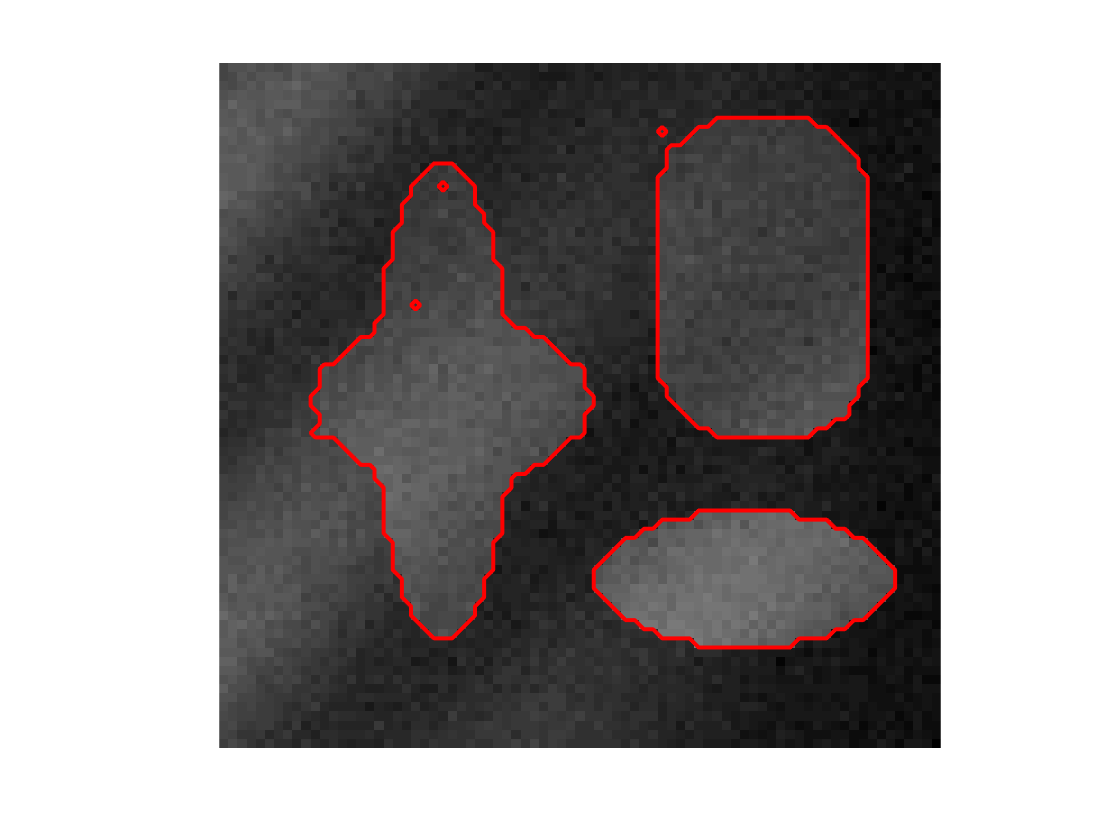}
\includegraphics[width=0.19\textwidth,clip,trim= 6cm 2cm 6cm 2cm]{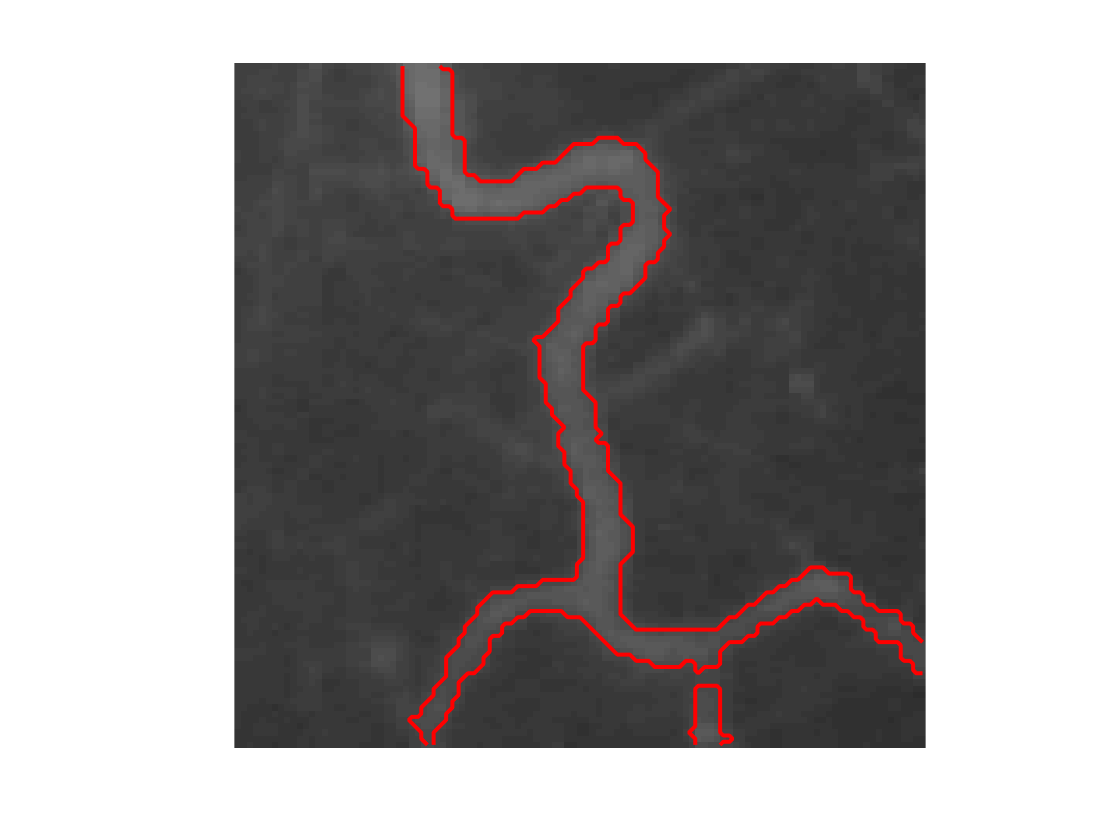}
\includegraphics[width=0.19\textwidth,clip,trim= 6cm 2cm 6cm 2cm]{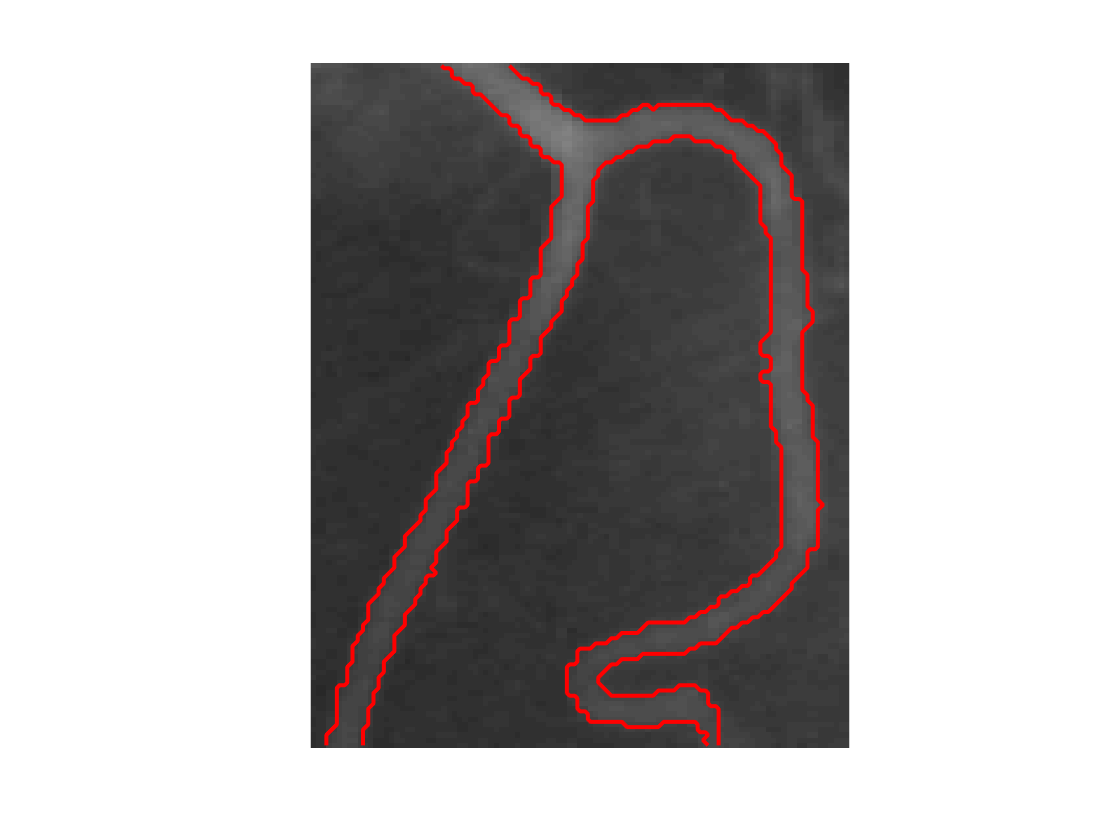}

\medskip

\begin{tabular}{|c|c|c|c|c|c|}
\hline
\# of iterations of the ICTM & 15  & 25 & 43& 28 & 47  \\ 
\hline
\# of iterations of the level-set method \cite{Chunming_Li_2008}& - & 256 & 131 & 117 & 209\\
\hline
\end{tabular}
\caption{Initial contour and segmented region using the ICTM in the LIF model. In all five experiments, $\mu_1 = \mu_2 = 1$.  From left to right: $(\sigma, \tau, \lambda) = (20,15, 500)$, $(3, 5, 150)$, $(3, 3, 245)$, $(3,10,110)$, and $(3,2,90)$. The results for the level-set method are obtained using the software code from \url{http://www.imagecomputing.org/~cmli/code/}. See Section~\ref{subsec:2} for details.} \label{fig:3}
\end{figure*}

\section{Conclusions and discussion} \label{sec:con}

In this paper, we proposed a novel iterative convolution-thresholding method (ICTM) that is applicable to a range of models for image segmentation. We considered the image segmentation as the minimization of a general energy functional consisting of a fidelity term of the image and a regularized term. The interfaces between different segments are implicitly determined by the characteristic functions of the segments. The fidelity term is then written into a linear functional in characteristic functions and the regularized term is approximated by a concave  functional of characteristic functions. We proved the energy-decaying property of the method.  Numerical experiments show that the method is simple, efficient, unconditionally stable, and insensitive to the number of segments.   The ICTM converges in significant fewer iterations than the level-set method for all the examples we tested. We expect that  the ICTM will be applicable to a large class of image segmentation models.

%for image segmentation, no level set function is used in ICTM and thus no regularization and reinitialization of the level set function is necessary. Also, the ICTM can be easily generalized to  the  arbitrary number of segments. Since we proved the unconditional stability of the ICTM, the result is more or less not sensitive to the choice of $\tau$. Numerical experiments show that we need many fewer steps of iteration to reach the stationary solution. 

%In this paper, we only performed numerical experiments to a few models. However, the ICTM can be applied to a large range of models for image segmentation. Also, we only consider the models whose regularization terms are perimeter of the boundary of each segments. We expect the ICTM can also be used for models with shape constraints such as isoperimetric constraints \cite{Gui_2017}. 

\section*{Acknowledgment}
This research was supported in part by the Hong Kong Research Grants Council GRF grants  16324416 and 16303318.

\appendices

%\section{Proof of Lemma \ref{lem:eq}}\label{append1}
%\DW{add proof here.}

\section{Proof of Theorem \ref{thm:stability}} \label{append2}
The proof consists of two parts: (1) to show that 
\begin{align}\mathcal{E}^\tau(u^{k+1},\Theta^{k})  \leq \mathcal{E}^\tau(u^{k},\Theta^{k})\label{stab11} \end{align}
and (2) to show that 
\begin{align}\mathcal{E}^\tau(u^{k+1},\Theta^{k+1})  \leq \mathcal{E}^\tau(u^{k+1},\Theta^{k}). \label{stab21} \end{align}
\eqref{stab21} is a direct consequence of \eqref{min:theta1}. Therefore we  only need to prove \eqref{stab11}.

To prove \eqref{stab11}, we write
\begin{align*} &\mathcal{L}^{\tau}(f,\Theta^k,u^k,u^k) \\
= &\mathcal{E}^{\tau}(u^k,\Theta^k) 
-  \frac{\lambda\sqrt{\pi}}{\sqrt{\tau}}\int_{\Omega} u^k  G_{\tau}*u^k \ dx 
\end{align*}
and 
\begin{align*} &\mathcal{L}^{\tau}(f,\Theta^k,u^k,u^{k+1})  
= \mathcal{E}^{\tau}(u^{k+1},\Theta^k) \\
&+ \frac{\lambda\sqrt{\pi}}{\sqrt{\tau}} \int_{\Omega} u^{k+1} G_{\tau}*(u^{k+1}-2u^k) \ dx.\end{align*} 
From \eqref{min1}, we have \[\mathcal{L}^{\tau}(f,\Theta^k,u^k,u^{k+1}) \leq \mathcal{L}^{\tau}(f,\Theta^k,u^k,u^k) .\] That is,
\begin{align}
&\mathcal{E}^{\tau}(u^{k+1},\Theta^k) \nonumber  \\
\leq &\mathcal{E}^{\tau}(u^k,\Theta^k)\label{stab3} \\
&- \frac{\lambda\sqrt{\pi}}{\sqrt{\tau}}\int_{\Omega} (u^k-u^{k+1}) G_\tau *(u^k-u^{k+1}) \ dx \nonumber \\
 = & \mathcal{E}^{\tau}(u^k,\Theta^k)-\frac{\lambda\sqrt{\pi}}{\sqrt{\tau}}  \int_{\Omega} \left[G_{\tau/2} *(u^k-u^{k+1})\right]^2 \ dx \nonumber\\
 \leq& \mathcal{E}^{\tau}(u^k,\Theta^k). \nonumber
\end{align}

\section{Proof of Theorem \ref{thm:stability2}} \label{append3}
Similar to the proof in Appendix~\ref{append2}, we only need to prove 
\begin{align}\mathcal{E}^\tau(u^{k+1},\Theta^{k})  \leq \mathcal{E}^\tau(u^{k},\Theta^{k}).\label{stab1} \end{align}
Again, we write
\begin{align*} &\mathcal{L}^{\tau}(f,\Theta^k,u^k,u^k) \\
= &\mathcal{E}^{\tau}(u^k,\Theta^k) 
+\frac{\lambda\sqrt{\pi}}{\sqrt{\tau}} \sum\limits_{i=1}^n\sum\limits_{j\neq i, j=1}^n \int_{\Omega} u_i^k  G_{\tau}*u_j^k \ dx 
\end{align*}
and 
\begin{align*} &\mathcal{L}^{\tau}(f,\Theta^k,u^k,u^{k+1})  
= \mathcal{E}^{\tau}(u^{k+1},\Theta^k) \\
&-  \frac{\lambda\sqrt{\pi}}{\sqrt{\tau}}\sum\limits_{i=1}^n\sum\limits_{j\neq i, j=1}^n\int_{\Omega} u_i^{k+1} G_{\tau}*u_j^{k+1} \ dx  \\
&+2\frac{\lambda\sqrt{\pi}}{\sqrt{\tau}} \sum\limits_{i=1}^n  \sum\limits_{j\neq i, j=1}^n\int_{\Omega}u_i^{k+1}G_{\tau}*u_j^k \ dx .\end{align*} 
From \eqref{min}, we have \[\mathcal{L}^{\tau}(f,\Theta^k,u^k,u^{k+1}) \leq \mathcal{L}^{\tau}(f,\Theta^k,u^k,u^k) .\] That is,
\begin{align}
&\mathcal{E}^{\tau}(u^{k+1},\Theta^k) \leq \mathcal{E}^{\tau}(u^k,\Theta^k) \label{stab3} \\
&+\sum\limits_{i=1}^{n}\sum\limits_{ j=1,j\neq i}^{n}\int_{\Omega} \frac{\lambda\sqrt{\pi}}{\sqrt{\tau}} (u_i^k-u_i^{k+1}) G_{\tau} *(u_j^k-u_j^{k+1}) \ dx. \nonumber
\end{align}
Direct calculation yields
\begin{align}
&\sum\limits_{i=1}^{n}\sum\limits_{ j=1,j\neq i}^{n}\int_{\Omega} \frac{\lambda\sqrt{\pi}}{\sqrt{\tau}} (u_i^k-u_i^{k+1}) G_{\tau} *(u_j^k-u_j^{k+1})\ dx \nonumber\\
 =&\sum\limits_{i=1}^{n}\int_{\Omega} \frac{\lambda\sqrt{\pi}}{\sqrt{\tau}} (u_i^k-u_i^{k+1}) G_{\tau} *\sum\limits_{ j=1,j\neq i}^{n}(u_j^k-u_j^{k+1})  dx\nonumber \\
=&\sum\limits_{i=1}^{n}\int_{\Omega} \frac{\lambda\sqrt{\pi}}{\sqrt{\tau}} (u_i^k-u_i^{k+1}) G_{\tau} *(u_i^{k+1}-u_i^k) \ dx   \nonumber  \\
=& - \sum\limits_{i=1}^{n}\int_{\Omega} \frac{\lambda\sqrt{\pi}}{\sqrt{\tau}}\left[G_{\tau/2} *(u_i^{k+1}-u_i^k) \right]^2\ dx\label{stab4} \\
 \leq & 0. \nonumber
\end{align}
Combining \eqref{stab3} and \eqref{stab4} gives
\[\mathcal{E}^{\tau}(u^{k+1},\Theta^k)  \leq  \mathcal{E}^{\tau}(u^k,\Theta^k).\]

\bibliographystyle{IEEEtran}
\bibliography{WW19.bib}

\begin{thebibliography}{10}
\providecommand{\url}[1]{#1}
\csname url@samestyle\endcsname
\providecommand{\newblock}{\relax}
\providecommand{\bibinfo}[2]{#2}
\providecommand{\BIBentrySTDinterwordspacing}{\spaceskip=0pt\relax}
\providecommand{\BIBentryALTinterwordstretchfactor}{4}
\providecommand{\BIBentryALTinterwordspacing}{\spaceskip=\fontdimen2\font plus
\BIBentryALTinterwordstretchfactor\fontdimen3\font minus
  \fontdimen4\font\relax}
\providecommand{\BIBforeignlanguage}[2]{{%
\expandafter\ifx\csname l@#1\endcsname\relax
\typeout{** WARNING: IEEEtran.bst: No hyphenation pattern has been}%
\typeout{** loaded for the language `#1'. Using the pattern for}%
\typeout{** the default language instead.}%
\else
\language=\csname l@#1\endcsname
\fi
#2}}
\providecommand{\BIBdecl}{\relax}
\BIBdecl

\bibitem{mitiche2010variational}
\BIBentryALTinterwordspacing
A.~Mitiche and I.~B. Ayed, \emph{Variational and Level Set Methods in Image
  Segmentation}.\hskip 1em plus 0.5em minus 0.4em\relax Springer Berlin
  Heidelberg, 2011. [Online]. Available:
  \url{https://doi.org/10.1007%2F978-3-642-15352-5}
\BIBentrySTDinterwordspacing

\bibitem{mumford1989optimal}
\BIBentryALTinterwordspacing
D.~Mumford and J.~Shah, ``Optimal approximations by piecewise smooth functions
  and associated variational problems,'' \emph{Communications on Pure and
  Applied Mathematics}, vol.~42, no.~5, pp. 577--685, jul 1989. [Online].
  Available: \url{https://doi.org/10.1002%2Fcpa.3160420503}
\BIBentrySTDinterwordspacing

\bibitem{ambrosio1990approximation}
\BIBentryALTinterwordspacing
L.~Ambrosio and V.~M. Tortorelli, ``Approximation of functional depending on
  jumps by elliptic functional via t-convergence,'' \emph{Communications on
  Pure and Applied Mathematics}, vol.~43, no.~8, pp. 999--1036, dec 1990.
  [Online]. Available: \url{https://doi.org/10.1002%2Fcpa.3160430805}
\BIBentrySTDinterwordspacing

\bibitem{chan2001active}
\BIBentryALTinterwordspacing
T.~F. Chan and L.~A. Vese, ``Active contours without edges,'' \emph{IEEE-IP},
  vol.~10, no.~2, pp. 266--277, 2001. [Online]. Available:
  \url{https://ieeexplore.ieee.org/document/902291/}
\BIBentrySTDinterwordspacing

\bibitem{vese2002multiphase}
\BIBentryALTinterwordspacing
L.~A. Vese and T.~F. Chan, ``A multiphase level set framework for image
  segmentation using the {Mumford and Shah} model,'' \emph{Int',I J.Computer
  vision}, vol.~50, no.~3, pp. 271--293, 2002. [Online]. Available:
  \url{https://doi.org/10.1023/A:1020874308076}
\BIBentrySTDinterwordspacing

\bibitem{Li_2007}
\BIBentryALTinterwordspacing
C.~Li, C.-Y. Kao, J.~C. Gore, and Z.~Ding, ``Implicit active contours driven by
  local binary fitting energy,'' in \emph{2007 {IEEE} Conference on Computer
  Vision and Pattern Recognition}.\hskip 1em plus 0.5em minus 0.4em\relax
  {IEEE}, jun 2007. [Online]. Available:
  \url{https://doi.org/10.1109%2Fcvpr.2007.383014}
\BIBentrySTDinterwordspacing

\bibitem{Chunming_Li_2008}
\BIBentryALTinterwordspacing
C.~Li, C.-Y. Kao, J.~Gore, and Z.~Ding, ``Minimization of region-scalable
  fitting energy for image segmentation,'' \emph{{IEEE} Transactions on Image
  Processing}, vol.~17, no.~10, pp. 1940--1949, oct 2008. [Online]. Available:
  \url{https://doi.org/10.1109%2Ftip.2008.2002304}
\BIBentrySTDinterwordspacing

\bibitem{Wang_2009}
\BIBentryALTinterwordspacing
L.~Wang, C.~Li, Q.~Sun, D.~Xia, and C.-Y. Kao, ``Active contours driven by
  local and global intensity fitting energy with application to brain {MR}
  image segmentation,'' \emph{Computerized Medical Imaging and Graphics},
  vol.~33, no.~7, pp. 520--531, oct 2009. [Online]. Available:
  \url{https://doi.org/10.1016%2Fj.compmedimag.2009.04.010}
\BIBentrySTDinterwordspacing

\bibitem{zhang2013local}
\BIBentryALTinterwordspacing
K.~Zhang, L.~Zhang, K.-M. Lam, and D.~Zhang, ``A local active contour model for
  image segmentation with intensity inhomogeneity,'' \emph{arXiv preprint},
  2013. [Online]. Available: \url{https://arxiv.org/abs/1305.7053}
\BIBentrySTDinterwordspacing

\bibitem{braides1998approximation}
\BIBentryALTinterwordspacing
A.~Braides, ``Approximation of free-discontinuity problems in
  \emph{Gamma-Convergence for Beginners}''.\hskip 1em plus 0.5em minus
  0.4em\relax Oxford University Press, jul 2002, pp. 121--131. [Online].
  Available:
  \url{https://doi.org/10.1093%2Facprof%3Aoso%2F9780198507840.003.0009}
\BIBentrySTDinterwordspacing

\bibitem{tsai2001curve}
\BIBentryALTinterwordspacing
A.~Tsai, A.~Yezzi, and A.~Willsky, ``Curve evolution implementation of the
  mumford-shah functional for image segmentation, denoising, interpolation, and
  magnification,'' \emph{{IEEE} Transactions on Image Processing}, vol.~10,
  no.~8, pp. 1169--1186, 2001. [Online]. Available:
  \url{https://doi.org/10.1109%2F83.935033}
\BIBentrySTDinterwordspacing

\bibitem{bae2011global}
\BIBentryALTinterwordspacing
E.~Bae, J.~Yuan, and X.-C. Tai, ``Global minimization for continuous multiphase
  partitioning problems using a dual approach,'' \emph{International Journal of
  Computer Vision}, vol.~92, no.~1, pp. 112--129, dec 2010. [Online].
  Available: \url{https://doi.org/10.1007%2Fs11263-010-0406-y}
\BIBentrySTDinterwordspacing

\bibitem{cai2013two}
\BIBentryALTinterwordspacing
X.~Cai, R.~Chan, and T.~Zeng, ``A two-stage image segmentation method using a
  convex variant of the mumford--shah model and thresholding,'' \emph{{SIAM}
  Journal on Imaging Sciences}, vol.~6, no.~1, pp. 368--390, jan 2013.
  [Online]. Available: \url{https://doi.org/10.1137%2F120867068}
\BIBentrySTDinterwordspacing

\bibitem{goldstein2009split}
\BIBentryALTinterwordspacing
T.~Goldstein and S.~Osher, ``The split bregman method for l1-regularized
  problems,'' \emph{{SIAM} Journal on Imaging Sciences}, vol.~2, no.~2, pp.
  323--343, jan 2009. [Online]. Available:
  \url{https://doi.org/10.1137%2F080725891}
\BIBentrySTDinterwordspacing

\bibitem{dong2010frame}
\BIBentryALTinterwordspacing
A.~Chien, B.~Dong, and Z.~Shen, ``Frame-based segmentation for medical
  images,'' \emph{Communications in Mathematical Sciences}, vol.~9, no.~2, pp.
  551--559, 2011. [Online]. Available:
  \url{https://doi.org/10.4310%2Fcms.2011.v9.n2.a10}
\BIBentrySTDinterwordspacing

\bibitem{esedog2006threshold}
\BIBentryALTinterwordspacing
S.~Esedoglu and Y.-H.~R. Tsai, ``Threshold dynamics for the piecewise constant
  mumford{\textendash}shah functional,'' \emph{Journal of Computational
  Physics}, vol. 211, no.~1, pp. 367--384, jan 2006. [Online]. Available:
  \url{https://doi.org/10.1016%2Fj.jcp.2005.05.027}
\BIBentrySTDinterwordspacing

\bibitem{bertozzi2012diffuse}
\BIBentryALTinterwordspacing
A.~L. Bertozzi and A.~Flenner, ``Diffuse interface models on graphs for
  classification of high dimensional data,'' \emph{{SIAM} Review}, vol.~58,
  no.~2, pp. 293--328, jan 2016. [Online]. Available:
  \url{https://doi.org/10.1137%2F16m1070426}
\BIBentrySTDinterwordspacing

\bibitem{merkurjevsemi}
\BIBentryALTinterwordspacing
E.~Merkurjev, A.~L. Bertozzi, and F.~Chung, ``A semi-supervised heat kernel
  pagerank {MBO} algorithm for data classification,'' \emph{Communications in
  Mathematical Sciences}, vol.~16, no.~5, pp. 1241--1265, 2018. [Online].
  Available: \url{https://doi.org/10.4310%2Fcms.2018.v16.n5.a4}
\BIBentrySTDinterwordspacing

\bibitem{merkurjev2014graph}
\BIBentryALTinterwordspacing
E.~Merkurjev, J.~Sunu, and A.~L. Bertozzi, ``Graph {MBO} method for multiclass
  segmentation of hyperspectral stand-off detection video,'' in \emph{2014
  {IEEE} International Conference on Image Processing ({ICIP})}.\hskip 1em plus
  0.5em minus 0.4em\relax {IEEE}, oct 2014. [Online]. Available:
  \url{https://doi.org/10.1109%2Ficip.2014.7025138}
\BIBentrySTDinterwordspacing

\bibitem{garcia2014multiclass}
\BIBentryALTinterwordspacing
C.~Garcia-Cardona, E.~Merkurjev, A.~L. Bertozzi, A.~Flenner, and A.~G. Percus,
  ``Multiclass data segmentation using diffuse interface methods on graphs,''
  \emph{{IEEE} Transactions on Pattern Analysis and Machine Intelligence},
  vol.~36, no.~8, pp. 1600--1613, aug 2014. [Online]. Available:
  \url{https://doi.org/10.1109%2Ftpami.2014.2300478}
\BIBentrySTDinterwordspacing

\bibitem{wang2016efficient}
\BIBentryALTinterwordspacing
D.~Wang, H.~Li, X.~Wei, and X.-P. Wang, ``An efficient iterative thresholding
  method for image segmentation,'' \emph{J. Comput. Phys.}, vol. 350, pp.
  657--667, 2017. [Online]. Available:
  \url{https://doi.org/10.1016/j.jcp.2017.08.020}
\BIBentrySTDinterwordspacing

\bibitem{merriman1992diffusion}
\BIBentryALTinterwordspacing
B.~Merriman, J.~K. Bence, and S.~Osher, \emph{Diffusion generated motion by
  mean curvature}.\hskip 1em plus 0.5em minus 0.4em\relax Department of
  Mathematics, University of California, Los Angeles, 1992. [Online].
  Available: \url{ftp://ftp.math.ucla.edu/pub/camreport/cam92-18.pdf}
\BIBentrySTDinterwordspacing

\bibitem{merriman1994motion}
\BIBentryALTinterwordspacing
B.~Merriman, J.~K. Bence, and S.~J. Osher, ``Motion of multiple junctions: A
  level set approach,'' \emph{Journal of Computational Physics}, vol. 112,
  no.~2, pp. 334--363, jun 1994. [Online]. Available:
  \url{https://doi.org/10.1006%2Fjcph.1994.1105}
\BIBentrySTDinterwordspacing

\bibitem{evans1993convergence}
\BIBentryALTinterwordspacing
L.~C. Evans, ``Convergence of an algorithm for mean curvature motion,''
  \emph{Indiana Math. J.}, vol.~42, no.~2, pp. 533--557, 1993. [Online].
  Available: \url{http://www.jstor.org/stable/24897106}
\BIBentrySTDinterwordspacing

\bibitem{esedoglu2015threshold}
\BIBentryALTinterwordspacing
S.~Esedoglu and F.~Otto, ``Threshold dynamics for networks with arbitrary
  surface tensions,'' \emph{Communications on Pure and Applied Mathematics},
  vol.~68, no.~5, pp. 808--864, jun 2014. [Online]. Available:
  \url{https://doi.org/10.1002%2Fcpa.21527}
\BIBentrySTDinterwordspacing

\bibitem{ruuth2003simple}
\BIBentryALTinterwordspacing
S.~J. Ruuth and B.~T. Wetton, ``A simple scheme for volume-preserving motion by
  mean curvature,'' \emph{J. Sci. Comput.}, vol.~19, no. 1-3, pp. 373--384,
  2003. [Online]. Available: \url{https://doi.org/10.1023/A:1025368328471}
\BIBentrySTDinterwordspacing

\bibitem{merkurjev2013mbo}
\BIBentryALTinterwordspacing
E.~Merkurjev, T.~Kosti{\'{c}}, and A.~L. Bertozzi, ``An {MBO} scheme on graphs
  for classification and image processing,'' \emph{{SIAM} Journal on Imaging
  Sciences}, vol.~6, no.~4, pp. 1903--1930, jan 2013. [Online]. Available:
  \url{https://doi.org/10.1137%2F120886935}
\BIBentrySTDinterwordspacing

\bibitem{merriman2000convolution}
\BIBentryALTinterwordspacing
B.~Merriman and S.~J. Ruuth, ``Convolution-generated motion and generalized
  {H}uygens' principles for interface motion,'' \emph{{SIAM J. on Appl.
  Math.}}, vol.~60, no.~3, pp. 868--890, 2000. [Online]. Available:
  \url{https://doi.org/10.1137/S003613999833397X}
\BIBentrySTDinterwordspacing

\bibitem{ruuth2001convolution}
\BIBentryALTinterwordspacing
S.~J. Ruuth and B.~Merriman, ``Convolution--thresholding methods for interface
  motion,'' \emph{J. Comput. Phys.}, vol. 169, no.~2, pp. 678--707, 2001.
  [Online]. Available: \url{http://dx.doi.org/10.1006/jcph.2000.6580}
\BIBentrySTDinterwordspacing

\bibitem{bonnetier2010consistency}
\BIBentryALTinterwordspacing
E.~Bonnetier, E.~Bretin, and A.~Chambolle, ``Consistency result for a non
  monotone scheme for anisotropic mean curvature flow,'' \emph{Interfaces and
  Free Boundaries}, vol.~14, no.~1, pp. 1--35, 2012. [Online]. Available:
  \url{dx.doi.org/10.4171/IFB/272}
\BIBentrySTDinterwordspacing

\bibitem{elsey2016threshold}
\BIBentryALTinterwordspacing
M.~Elsey and S.~Esedoḡlu, ``Threshold dynamics for anisotropic surface
  energies,'' \emph{Mathematics of Computation}, vol.~87, no. 312, pp.
  1721--1756, oct 2017. [Online]. Available:
  \url{https://doi.org/10.1090%2Fmcom%2F3268}
\BIBentrySTDinterwordspacing

\bibitem{xu2016efficient}
\BIBentryALTinterwordspacing
X.~Xu, D.~Wang, and X.-P. Wang, ``An efficient threshold dynamics method for
  wetting on rough surfaces,'' \emph{Journal of Computational Physics}, vol.
  330, pp. 510--528, feb 2017. [Online]. Available:
  \url{https://doi.org/10.1016%2Fj.jcp.2016.11.008}
\BIBentrySTDinterwordspacing

\bibitem{wang2018efficient}
\BIBentryALTinterwordspacing
D.~Wang, X.-P. Wang, and X.~Xu, ``An improved threshold dynamics method for
  wetting dynamics,'' \emph{submited}, 2018. [Online]. Available:
  \url{https://www.math.utah.edu/~dwang/threshold_modified_2.pdf}
\BIBentrySTDinterwordspacing

\bibitem{chen2018efficient}
\BIBentryALTinterwordspacing
H.~Chen, H.~Leng, D.~Wang, and X.-P. Wang, ``An efficient threshold dynamics
  method for topology optimization for fluids,'' \emph{arXiv preprint}, 2018.
  [Online]. Available: \url{https://arxiv.org/abs/1812.09437}
\BIBentrySTDinterwordspacing

\bibitem{wang2018dynamics}
\BIBentryALTinterwordspacing
D.~Wang, A.~Cherkaev, and B.~Osting, ``Dynamics and stationary configurations
  of heterogeneous foams,'' \emph{arXiv preprint}, 2018. [Online]. Available:
  \url{https://arxiv.org/abs/1811.03570}
\BIBentrySTDinterwordspacing

\bibitem{Gennip2013}
\BIBentryALTinterwordspacing
Y.~van Gennip, N.~Guillen, B.~Osting, and A.~L. Bertozzi, ``Mean curvature,
  threshold dynamics, and phase field theory on finite graphs,'' \emph{Milan
  Journal of Mathematics}, vol.~82, no.~1, pp. 3--65, 2014. [Online].
  Available: \url{http://dx.doi.org/10.1007/s00032-014-0216-8}
\BIBentrySTDinterwordspacing

\bibitem{jacobsauction}
\BIBentryALTinterwordspacing
M.~Jacobs, E.~Merkurjev, and S.~Esedo{\=g}lu, ``Auction dynamics: A volume
  constrained {MBO} scheme,'' \emph{Journal of Computational Physics}, vol.
  354, pp. 288--310, feb 2018. [Online]. Available:
  \url{https://doi.org/10.1016%2Fj.jcp.2017.10.036}
\BIBentrySTDinterwordspacing

\bibitem{deckelnick2005computation}
\BIBentryALTinterwordspacing
K.~Deckelnick, G.~Dziuk, and C.~M. Elliott, ``Computation of geometric partial
  differential equations and mean curvature flow,'' \emph{Acta numerica},
  vol.~14, pp. 139--232, 2005. [Online]. Available:
  \url{https://doi.org/10.1017/S0962492904000224}
\BIBentrySTDinterwordspacing

\bibitem{esedoglu2008threshold}
\BIBentryALTinterwordspacing
S.~Esedoglu, S.~Ruuth, and R.~Tsai, ``Threshold dynamics for high order
  geometric motions,'' \emph{Interfaces and Free Boundaries}, vol.~10, no.~3,
  pp. 263--282, 2008. [Online]. Available:
  \url{http://dx.doi.org/10.4171/IFB/189}
\BIBentrySTDinterwordspacing

\bibitem{ishii2005optimal}
\BIBentryALTinterwordspacing
K.~Ishii, ``Optimal rate of convergence of the {B}ence-{M}erriman-{O}sher
  algorithm for motion by mean curvature,'' \emph{SIAM J. Math. Anal.},
  vol.~37, no.~3, pp. 841--866, 2005. [Online]. Available:
  \url{https://doi.org/10.1137/04061862X}
\BIBentrySTDinterwordspacing

\bibitem{ruuth1998diffusion}
\BIBentryALTinterwordspacing
S.~J. Ruuth, ``A diffusion-generated approach to multiphase motion,'' \emph{J.
  Comput. Phys.}, vol. 145, no.~1, pp. 166--192, 1998. [Online]. Available:
  \url{https://doi.org/10.1006/jcph.1998.6028}
\BIBentrySTDinterwordspacing

\bibitem{ruuth1998efficient}
\BIBentryALTinterwordspacing
S.~Ruuth, ``Efficient algorithms for diffusion-generated motion by mean
  curvature,'' \emph{Journal of Computational Physics}, vol. 144, no.~2, pp.
  603--625, aug 1998. [Online]. Available:
  \url{https://doi.org/10.1006%2Fjcph.1998.6025}
\BIBentrySTDinterwordspacing

\bibitem{nufft2}
\BIBentryALTinterwordspacing
A.~Dutt and V.~Rokhlin, ``Fast {F}ourier transforms for nonequispaced data,''
  \emph{SIAM J. Sci. Comput.}, vol.~14, pp. 1368--1393, 1993. [Online].
  Available: \url{https://doi.org/10.1137/0914081}
\BIBentrySTDinterwordspacing

\bibitem{nufft6}
\BIBentryALTinterwordspacing
J.~Y. Lee and L.~Greengard, ``The type 3 nonuniform {FFT} and its
  applications,'' \emph{J. Comput. Phys.}, vol. 206, pp. 1--5, 2005. [Online].
  Available: \url{https://doi.org/10.1016/j.jcp.2004.12.004}
\BIBentrySTDinterwordspacing

\bibitem{jiang2016nufft}
\BIBentryALTinterwordspacing
S.~Jiang, D.~Wang, and X.~Wang, ``{An efficient boundary integral scheme for
  the MBO threshold dynamics method via NUFFT},'' \emph{To appear in J. Sci.
  Comput.}, 2017. [Online]. Available:
  \url{https://doi.org/10.1007/s10915-017-0448-1}
\BIBentrySTDinterwordspacing

\bibitem{wang2018diffusion}
\BIBentryALTinterwordspacing
D.~Wang and B.~Osting, ``A diffusion generated method for computing {Dirichlet}
  partitions,'' \emph{Journal of Computational and Applied Mathematics}, vol.
  351, pp. 302--316, may 2019. [Online]. Available:
  \url{https://doi.org/10.1016%2Fj.cam.2018.11.015}
\BIBentrySTDinterwordspacing

\bibitem{osting2017generalized}
\BIBentryALTinterwordspacing
B.~Osting and D.~Wang, ``A generalized {MBO} diffusion generated motion for
  orthogonal matrix-valued fields,'' \emph{arXiv preprint}, 2017. [Online].
  Available: \url{https://arxiv.org/abs/1711.01365}
\BIBentrySTDinterwordspacing

\bibitem{osting2018}
\BIBentryALTinterwordspacing
B. Osting and D. Wang, ``Diffusion generated methods for denoising target-valued images,''
  \emph{arXiv preprint}, 2018. [Online]. Available:
  \url{https://arxiv.org/abs/1806.06956}
\BIBentrySTDinterwordspacing

\bibitem{Miranda_2007}
\BIBentryALTinterwordspacing
M.~Miranda, D.~Pallara, F.~Paronetto, and M.~Preunkert, ``Short-time heat flow
  and functions of bounded variation in $\mathbf{R}^n$,'' \emph{Annales de la
  facult{\'{e}} des sciences de Toulouse Math{\'{e}}matiques}, vol.~16, no.~1,
  pp. 125--145, 2007. [Online]. Available:
  \url{https://doi.org/10.5802%2Fafst.1142}
\BIBentrySTDinterwordspacing

\bibitem{wang_2001}
\BIBentryALTinterwordspacing
X.-P. Wang, C.~J. Garcia-Cervera, and W.~E, ``A {Gauss--Seidel} projection
  method for micromagnetics simulations,'' \emph{Journal of Computational
  Physics}, vol. 171, no.~1, pp. 357--372, 7 2001. [Online]. Available:
  \url{https://doi.org/10.1006%2Fjcph.2001.6793}
\BIBentrySTDinterwordspacing

\end{thebibliography}

\begin{IEEEbiography}{Dong Wang}
Dr. Wang received his bachelor's degree in mathematics from Sichuan University in 2013 and his Ph.D. degree in mathematics from the Hong Kong University of Science and Technology in 2017. He is now a postdoc fellow in the Department of Mathematics at the University of Utah.  
\end{IEEEbiography}

\begin{IEEEbiography}{Xiao-Ping Wang}
Prof.  Wang received his bachelor's degree in mathematics from Peking University in 1984 and his Ph.D. degree in mathematics from the Courant Institute (New York University) in 1990. He was a postdoctoral at MSRI in Berkeley, University of Colorado in Boulder before he came to HKUST.  Prof. Wang is currently the Head and  Chair Professor of Mathematics at the Hong Kong University of Science and Technology.    He received Feng Kang Prize of Scientific Computing in 2007 and was a plenary speaker at the SIAM conference on mathematical aspects of materials science 2016 and an invited seaker at the ICIAM 2019.
\end{IEEEbiography}

%\begin{IEEEbiographynophoto}{Dong Wang}
%% insert where needed to balance the two columns on the last page with
%%\newpage
%
%\begin{IEEEbiographynophoto}{Xiao-Ping Wang}


%\vfill


\end{document}